%% file: main.tex
\pdfoutput=1
\documentclass[11pt]{article}
\usepackage[margin=1in]{geometry}
\usepackage{times}
\usepackage{amsmath,amssymb,amsthm}
\newtheorem{theorem}{Theorem}
\newtheorem{lemma}{Lemma}
\newtheorem{proposition}{Proposition}
\newtheorem{corollary}{Corollary}

\newtheorem{assumption}{Assumption}
\theoremstyle{definition}

\newtheorem{definition}{Definition}

\theoremstyle{plain}

\newcommand{\BFr}{\mathbf{r}}
\newcommand{\BFm}{\mathbf{m}}
\newcommand{\BFg}{\mathbf{g}}
\newcommand{\BFM}{\mathbf{M}}
\def\halmos{\mbox{\quad$\square$}}
\def\argmax{\mathop{\rm arg\,max}}
\providecommand{\OneAndAHalfSpacedXII}{}

\newenvironment{APPENDICES}{\appendix}{}
\makeatletter
\@for\@bfl:={a,b,c,d,e,f,g,h,i,j,k,l,m,n,o,p,q,r,s,t,u,v,w,x,y,z}\do{%
  \expandafter\xdef\csname BF\@bfl\endcsname{\noexpand\mathbf{\@bfl}}}
\@for\@bfu:={A,B,C,D,E,F,G,H,I,J,K,L,M,N,O,P,Q,R,S,T,U,V,W,X,Y,Z}\do{%
  \expandafter\xdef\csname BF\@bfu\endcsname{\noexpand\mathbf{\@bfu}}}
\makeatother
\usepackage{times}
\RequirePackage{bm}
\RequirePackage{endnotes}
\OneAndAHalfSpacedXII 
\usepackage{algorithm}
\usepackage{algpseudocode}
\usepackage{longtable}
\usepackage{tikz}
\usepackage{personal}

\usepackage{nomencl}
\usepackage{etoolbox}
\makenomenclature

\renewcommand{\nomgroup}[1]{%
\item[\bfseries
\ifstrequal{#1}{P}{Problem Setup}{%
\ifstrequal{#1}{B}{Bayesian / $\alpha$-Posterior}{%
\ifstrequal{#1}{R}{Assumptions and Regularity}{%
\ifstrequal{#1}{A}{Algorithm and Analysis}{}}}}%
]}

\usepackage[sort&compress]{natbib}
 \bibpunct[, ]{[}{]}{,}{n}{}{,}%

\begin{document}


%

\title{Posterior Tempering Explains Variance Inflation in Linear and Generalized Linear Thompson Sampling}
\author{%
Prateek Jaiswal\thanks{Daniels School of Business, Purdue University. Email: \texttt{jaiswalp@purdue.edu}}
\and
Debdeep Pati\thanks{Department of Statistics, University of Wisconsin-Madison. Email: \texttt{dpati2@wisc.edu}}
\and
Anirban Bhattacharya\thanks{Department of Statistics, Texas A\&M University. Email: \texttt{anirbanb@tamu.edu}}
\and
Bani K. Mallick\thanks{Department of Statistics, Texas A\&M University. Email: \texttt{bkmallick@tamu.edu}}
}
\date{}
\maketitle
\begin{abstract}
We study a variant of the Thompson Sampling~(TS) algorithm, called $\alpha$-TS, for solving stochastic generalized linear bandit problems.  Existing analyses of TS require inflating the posterior variance to derive near-optimal regret guarantees. We formalize the idea of variance inflation by introducing $\alpha$-TS that uses a fractional or $\alpha$-posterior instead of the standard posterior. 
Our main contribution is to identify general regularity conditions on the prior and reward distributions that enable a regret analysis of $\alpha$-TS without assuming any tractable approximation of the posterior distribution, unlike previous works. For a specific choice of $\alpha \propto d^{-1}$, our general regret bound yields the best known regret bound of $O(d^{3/2}\sqrt{T}\log T)$ for both the exponential and sub-Gaussian families of reward distributions. 
We further provide an $\alpha$-dependent lower bound showing that the regret constant depends on the product $\alpha d$, and that when $\alpha \propto  d^{-1}$ the regret scales as $\Omega(d^{3/2}\sqrt{T})$, explaining the origin of the $d^{3/2}$ factor in the upper bound. 
Our proof technique  adapts and combines recent advancements in the analysis of linear bandit problems with first- and second-order posterior concentration theory from the Bayesian statistics literature.
\end{abstract}

\noindent\textbf{Keywords:} Thompson sampling, generalized linear models, frequentist regret bounds, linear bandits, finite-sample Bernstein-von Mises, posterior concentration

\input{Intro}

\input{LitRev}
\input{Problem}
\input{Assumptions}
\input{LowerBound}
\input{Analysis}
\input{Conclusion}

\begingroup \parindent 0pt \parskip 4ex
\def\enotesize{\normalsize} 
\theendnotes
\endgroup
\bibliographystyle{plainnat} 
\bibliography{refs} 

 \begin{APPENDICES}
 \input{Nomenclature}

\input{Supp}
 \end{APPENDICES}

\end{document}

%% file: Intro.tex
\section{Introduction}
The bandit framework~\citep{Robins1952} is widely used to model various sequential decision-making problems with many applications in healthcare, advertising, resource allocation, robotics, material discovery, etc. A central challenge in such problems is balancing exploration and exploitation. To solve such sequential decision-making problems, the decision-maker must balance exploration and exploitation.  In particular, the decision-maker must choose between exploring less-understood actions to gain new information and exploiting existing knowledge to select the (statistically) best action. This work considers a stochastic generalized linear bandit problem with compact decision space ($\sA$ $\subset \R^d$) and reward distribution with generalized linear mean.

Thompson Sampling (TS)~\citep{THOMPSON1933} and \textit{optimism-in-the-face-of-uncertainty} (OFU) \citep{lai1985asymptotically} are two broad classes of algorithms to solve various versions of the bandit problem. The OFU algorithms compute a confidence set for the unknown coefficients in the reward model and then select an optimistic estimate of the coefficient and decision that maximizes the estimated outcome. TS instead uses a Bayesian heuristic (posterior distribution) to address the exploration-exploitation dilemma.  TS samples model parameters from the posterior and selects actions that are optimal under the sampled parameters. The theoretical performance of such algorithms is typically measured by computing a bound on a term called regret, which is the total sum of the loss incurred by the algorithm over past trials.  At each round, regret arises from selecting a suboptimal action instead of the oracle's optimal action. In this work, we theoretically analyze a version of TS to solve the generalized linear bandit (LB) problem.

Agrawal and Goyal~\cite{agrawal2014thompsonsamplingcontextualbandits}, in their pioneering work, analyze TS for LB problems and derive a state-of-the-art regret bound of $O(d^{3/2}\sqrt{T}\log T)$, where $T$ is the number of trials. 
Their regret bound is very close to the lower bound of $\Omega(d\sqrt{T})$ for LB derived in~\cite{Dani2008}. The other important work by~\cite{abeille2017linear} computes a similar upper bound by using a novel proof technique.  
Both works consider a TS algorithm for LB with sub-Gaussian errors; however, the sampling distribution coincides with the posterior distribution only when the errors are centered Gaussian and the prior distribution on the unknown coefficients is also Gaussian (conjugate setting). The closed-form expression of the Gaussian sampling distribution enables them to compute anti-concentration bounds; the availability of which is a crucial component for the analysis in both works. However, recall that the Bayesian heuristic in TS is more general as it is applicable for any prior and reward combination~\citep{li2012open,chapelle2011empirical,urteaga2018nonparametric,hong22b}. This work contributes to further generalizing the analysis of TS with a more general reward and prior model  while maintaining sampling from the Bayesian posterior rather than relying on Gaussian approximations. Our regularity conditions on the reward distributions are satisfied by both sub-Gaussian and exponential families. While the sub-Gaussian family has been extensively used in most of the previous works on LB~\citep{Dani2008,abbasi2011improved,agrawal2013thompsonLIN,abeille2017linear}, the use of the exponential family (generalized linear model (GLM)~\citep{Nelder1972}) is novel to the best of our knowledge.

We study a version of the TS algorithm, termed as $\alpha$-TS, that uses fractional or $\alpha$-posterior~\citep{bhattacharya2019bayesian} instead of the standard posterior distribution. To compute the $\alpha$-posterior distribution, the likelihood in the definition of the standard posterior distribution is tempered with a factor $\alpha$.  
For $\alpha$-TS, we compute a regret bound for any $\alpha\in (0,1)$, but it grows exponentially with $d$. However, by optimally choosing $\alpha$ for the exponential and sub-Gaussian families of distribution, we show that the regret bound is of $O(d^{3/2}\sqrt{T}\log T)$ for $\alpha=d^{-1}$. 
Intuitively, setting $\alpha$ to a value in $(0,1)$ inflates the variance of the posterior distribution. A similar variance inflation factor is also considered in the sampling distribution (or posterior distribution) for the analysis in~\cite{agrawal2013thompsonLIN} and~\cite{abeille2017linear}. The $\alpha$-TS formulation provides a principled way to formalize this variance inflation through fractional posterior sampling. Moreover, we remove the dependence on a hyperparameter $\delta\in(0,1)$ from the TS algorithm and derive regret bounds in expectation, as conjectured by~\cite{agrawal2013thompsonLIN}.

Our focus on $\alpha$-TS is primarily motivated by theoretical considerations. In many analyses of TS for linear and contextual bandits~\citep{agrawal2013thompsonLIN,abeille2017linear}, the sampling distribution is effectively inflated relative to the standard posterior in order to ensure a constant probability of optimism. This inflation typically appears as an additional scaling factor in the covariance and is introduced \emph{ad hoc} in the analysis. The $\alpha$-TS formulation provides a principled way to formalize this phenomenon through {fractional posteriors}~\citep{bhattacharya2019bayesian}: the tempering parameter $\alpha \in (0,1)$ induces a posterior distribution with inflated covariance while preserving a coherent Bayesian interpretation. Importantly, $\alpha$-TS does not introduce a new algorithmic procedure—the resulting method coincides with variance-inflated Linear Thompson Sampling (LinTS) as studied in~\cite{agrawal2013thompsonLIN,abeille2017linear}. This perspective unifies existing variance-inflation heuristics within a coherent probabilistic framework. Our contribution is therefore theoretical: we provide a principled statistical interpretation of variance inflation in TS and develop a general regret analysis framework that applies under broad structural conditions on the prior and reward model.

The role of $\alpha$ is particularly important in the regret analysis. Without variance inflation, the posterior concentrates rapidly, making optimistic samples increasingly unlikely as the dimension grows. This leads to regret bounds with an undesirable exponential dependence on $d$. Choosing $\alpha \le d^{-1}$ inflates the sampling covariance sufficiently to maintain a constant probability of optimism, effectively counteracting this concentration and eliminating the exponential dependence. This yields the near-optimal $O(d^{3/2}\sqrt{T}\log T)$ regret rate. We further clarify this phenomenon in the analysis and provide a complementary lower bound showing how the regret depends explicitly on the product $\alpha d$ through the tempering of the posterior.

A closely related work by~\cite{jaiswal2022} also considers $\alpha$-posterior for TS, but they focus on studying the MAB problem. Also, it is unclear whether their analysis can be extended to a linear bandit setting. While the $K$-armed bandit can be viewed as a special case of the linear bandit obtained by restricting the action set to the canonical basis vectors, it is well known that regret guarantees for linear bandits do not automatically recover the sharp bounds in the MAB setting. This is because the analyses rely on fundamentally different structural properties: in MAB, the uncertainty of each arm decreases directly with the number of pulls, whereas in linear bandits the uncertainty is governed by the geometry of the design matrix. As discussed in~\cite{agrawal2013thompsonLIN}, this mismatch prevents linear bandit analyses from recovering MAB-optimal rates. By substituting $\sA$ with a $K$-dim basis (to convert it into a K-arm Bandit problem), our bound does not recover the bounds in~\cite{jaiswal2022} (same as the bounds in~\cite{agrawal2013thompsonLIN} do not recover the state-of-the-art bounds for the MAB setting). Moreover, the bounds in~\cite{jaiswal2022} hold for any $\alpha\in(0,1)$, whereas here, $\alpha$ must scale as $d^{-1}$ to derive optimal dimension-dependent bounds.

Our regret analysis builds on the idea of partitioning actions into saturated and unsaturated sets, as introduced in~\cite{agrawal2013thompsonLIN}. However, we provide a different set of proof arguments to derive regret bounds in expectation rather than in high probability. Computing the regret bounds requires deriving both posterior concentration and anti-concentration properties. Unlike previous works~\cite{agrawal2013thompsonLIN,abeille2017linear}, which derive these concentration and anti-concentration properties using closed-form expressions of the sampling distribution (e.g., using Mill's ratio for Gaussian density), we develop a more general approach. Specifically, we adapt ideas from the Bayesian statistics literature~\cite{bhattacharya2019bayesian,ZG,GGV,Spokoiny2012,Panov2015} to establish these properties under very general conditions on the prior and the reward model. While deriving the first-order concentration results for the $\alpha$-posterior, we redefine the conventional metric (which measures the distance of a model parameter from the true parameter) to align with the metric used to measure regret in the linear bandit analysis. To derive anti-concentration properties for any reward and prior distribution without relying on closed-form expressions, we adapt finite-sample version of the second-order posterior concentration results, popularly referred to as the Bernstein-von Mises theorem, as developed in~\cite{Spokoiny2012,Panov2015}. Notably, our results hold under very general conditions on the prior and reward model, enabling a broad generalization of the regret analysis for $\alpha$-TS.

%% file: LitRev.tex
\section{Literature review}

In addition to the works discussed in the introduction, there are various research directions that recent works have pursued for TS to solve LB problems.~\cite{luo_geometry-aware_2023,hamidi_frequentist_2023} focus on improving the regret bound for TS so that it matches the lower bound of $\Omega(d\sqrt{T})$. In particular,~\cite{luo_geometry-aware_2023} proposes a new algorithm that switches between OFU and TS algorithms based on the history of observations and computes minimax optimal frequentist regret guarantees.
\cite{Zhang2021feel} proposes Feel-Good Thompson Sampling and studies the theoretical properties of Thompson Sampling without imposing structural assumptions on the posterior sampling distribution. In particular, \cite{Zhang2021feel} introduces a modified likelihood in which the negative reward log-likelihood is defined as a weighted squared difference between the true reward and a sampled reward. This modified likelihood favors sampling of optimistic models from the posterior distribution and enables the author to derive regret bounds that match the lower bound for linear bandits. However, the construction in~\cite{Zhang2021feel} relies on a finite model class where the suboptimal action produces identical rewards under all candidate models---repeatedly selecting that action yields no information and the posterior remains unchanged. This degeneracy violates the regularity conditions we impose, which require the existence of informative actions that distinguish nearby parameters and allow posterior concentration. 

 There is also increasing interest in using more complex reward models, such as semi-parametric~\citep{greenewald2017action},~nonparametric~\citep{rigollet2010nonparametric,kim2021multi}, and high-dimensional~\citep{bastani2020online,chakraborty2023thompson} reward models in contextual bandit settings.

Several recent works propose frequentist algorithms that achieve strong regret guarantees for contextual bandits~\citep{hamidi_frequentist_2023}, and~\cite{kim2023double} proposes the DDRTS algorithm for the GLM bandit setting. A key distinction between these works and ours is the structure of the action set. These results are typically developed for the \emph{finite-arm} contextual bandit setting, where regret bounds depend on the number of arms $K$, and the analysis exploits concentration together with union bounds over the finite action set. Our work, by contrast, considers the linear bandit formulation where the action set is \emph{infinite} (e.g., a compact subset $\mathcal{X}\subset\mathbb{R}^d$), requiring uniform control over an uncountable family of arms---a substantially different analytical challenge. Moreover, as emphasized above, $\alpha$-TS does not introduce a new algorithmic procedure: it coincides with variance-inflated LinTS, and our contribution lies in the general theoretical framework it enables for analyzing posterior-sampling algorithms under structural conditions on the prior and reward model.

%% file: Problem.tex
\section{Problem setup}
\nomenclature[P]{$\sA$}{Action (decision) space, $\sA \subset \mathbb{R}^d$}
\nomenclature[P]{$d$}{Dimension of the parameter space}
\nomenclature[P]{$T$}{Time horizon (number of rounds)}
\nomenclature[P]{$[n]$}{Set $\{1,2,\dots,n\}$}
\nomenclature[P]{$a_t$}{Action selected at time $t$}
\nomenclature[P]{$r_t$}{Observed reward at time $t$}
\nomenclature[P]{$\theta_0$}{True (unknown) parameter}
\nomenclature[P]{$\Theta$}{Parameter space}
\nomenclature[P]{$P_{\theta}(\cdot \protect \mid a)$}{Reward distribution given action $a$ under parameter $\theta$}
\nomenclature[P]{$p_{\theta}(\cdot \protect \mid a)$}{Conditional reward density under parameter $\theta$}
\nomenclature[P]{$X_t$}{History of observations $\{a_s,r_s\}_{s=1}^t$}
\nomenclature[P]{$\gF_t$}{Filtration representing information up to time $t$}
\nomenclature[P]{$q(\cdot \protect \mid \gF_t)$}{Randomized action selection policy}
\nomenclature[P]{$g(\cdot)$}{Link function (monotone and Lipschitz)}
\nomenclature[P]{$a_0$}{Oracle optimal action}
\nomenclature[P]{$\mReg(T)$}{Cumulative regret up to time $T$}
We consider a stochastic generalized linear bandit problem. We denote the arbitrary action space as $\sA \subset \R^d$ 
and the time horizon as $T$.
For any $n\in \sN$, we denote $[n]$ to represent $\{1,2,3,\ldots, n \}$. In this problem, at each time step $t\in[T]$, the learner chooses an action $a_t\in \sA$ (according to some rule, possibly randomized) and observes a reward $r_t$ as feedback from an unknown environment in response to the learner's action. 
We model the environment by assuming that the reward at time $t$ is generated from a distribution $P_{\theta_0}(\cdot|a_t)$ with conditional density function denoted as $p_{0}(\cdot|a_t)$. 
Here, $\theta_0 \in \Theta\subseteq \R^d$ is an unknown but fixed parameter. 
While interacting with the environment, the learner collects the history of observations, which we denote as $X_t:= \{a_s,r_s\}_{s=1}^{t}$. We define the filtration generated by $X_t$ as $\gF_t := \gF_0\cup\{\sigma(X_t)\}$, that contains all the information till time $t$ including any prior information $\gF_0$. At time $t+1$, the learner takes an action according to a randomized action selection policy $q(\cdot|\gF_t)$, that is $a_{t+1}\sim q(\cdot|\gF_t)$.
We also assume that the conditional mean reward under distribution $P_{\theta}(\cdot|a_t)$ can be represented using a strictly monotonic and Lipschitz link function $g:\R \mapsto \R$ as $g(a_t^\top \theta)$. 

We evaluate a learner (or algorithm) by comparing its action at time $t$ to the oracle best action $a_0:= \argmax_{a\in \sA} g(a^\top \theta_0)$.  In particular, the learner's objective is to minimize the following cumulative regret on a sample path of observations,
\begin{align}\label{eq:Regret}
    \mReg(T) =  \sum_{t=1}^{T} (g(a_0^\top \theta_0)- g(a_t^\top \theta_0)),
\end{align}
which quantifies the total regret of taking suboptimal action $a_t$ instead of the oracle's best action $a_0$. We measure the algorithm's performance by analyzing the $\E[\mReg(T)]$,  where the expectation is taken with respect to the distribution that generates the sequence of observations $X_{T-1}$. Unless stated otherwise, $\E[\cdot]$ denotes expectation with respect to the data-generating distribution. 

\subsection{$\alpha$-Thompson Sampling}
We study a variant of the standard Thompson sampling (TS) algorithm, where we use $\alpha$-posterior instead of the standard posterior distribution ($\alpha=1$). We call this new algorithm as $\alpha$-TS. Recall that TS was proposed by~\cite{THOMPSON1933} to solve a multiarmed bandit problem, where the exploration-exploitation dilemma is addressed using a posterior distribution. The pseudo-code of $\alpha$-TS is provided in Algorithm~\ref{alg:TS}. In $\alpha$-TS, we posit a prior distribution $\Pi(\cdot)$ on a measurable space $(\Theta,\mathcal{Q})$. For any $t\geq 1$ and $\alpha\in(0,1)$, we define the $\alpha$-posterior distribution $\Pi_{t,\alpha}(\cdot)$ for any measurable set $Q\in \mathcal{Q}$ as
\begin{align}
    \Pi_{t,\alpha}(Q|\gF_t) &= \frac{ \int_Q \prod_{s=1}^{t} [p_{\theta}(r_s|a_s)q(a_s|\gF_{s-1})]^{\alpha} \Pi(d\theta)  } { \int_{\Theta}  \prod_{s=1}^{t} [p_{\theta}(r_s|a_s)q(a_s|\gF_{s-1})]^{\alpha} \Pi(d\theta) }
        = \frac{ \int_Q \prod_{s=1}^{t} [p_{\theta}(r_s|a_s)]^{\alpha} \Pi(d\theta)  } { \int_{\Theta}  \prod_{s=1}^{t} [p_{\theta}(r_s|a_s)]^{\alpha} \Pi(d\theta) },
        \label{eq:AlphaPost}
\end{align}
 where $q(a_s|\gF_{s-1})$ is the likelihood (or the randomized action selection rule) of choosing $a_s\in \sA$ given $\gF_{s-1}$ and $p_{\theta}(r_s|a_s)$ is the likelihood of observing $r_s$ given $a_s$. Note that $a_s$ is adapted to $\gF_{s-1}$. Essentially, step $3$ and $4$ of~Algorithm~\ref{alg:TS}, combines together to sample from the randomized action selection rule $q(a_s|\gF_{s-1})$. 
 In case if the action space $\sA$ is a discrete set, then $q(a_s|\gF_{s-1})$ can be precisely defined as the $\alpha$-posterior probability of choosing $a_s$ as the optimal action,
 that is equals to $\Pi_{t,\alpha}(\left\{a_s^\top\theta = \max_{a\in\sA} a^\top\theta \right\}|\gF_{s-1})$.  
 Observe that the definition of the $\alpha$-posterior  distribution in~\eqref{eq:AlphaPost} is unaffected by the probability of choosing the best action $a_s$. In the rest of the paper, we write $E_\alpha[\cdot|\gF_t]$ for expectation with respect to the $\alpha$-posterior distribution.

\nomenclature[B]{$E_{\alpha}[\cdot \protect \mid \gF_t]$}{Expectation under $\alpha$-posterior given $\gF_t$}
\nomenclature[B]{$P_{\alpha}(\cdot \protect \mid \gF_t)$}{Probability under $\alpha$-posterior given $\gF_t$}
\begin{algorithm}[b]
\caption{$\alpha$-Thompson Sampling ($\alpha$-TS)}\label{alg:TS}
\vskip6pt
\begin{algorithmic}
\State{Set $X_0=\{\}$, $g(\cdot)$, and prior distribution $\Pi$ on $\Theta$}
        \For{time step $t=1,2,\ldots T$}
        \State Update $\alpha$-posterior $\Pi_{t,\alpha}(\cdot|\gF_{t-1})$\,
        \State Generate a sample $\theta_t \sim \Pi_{t,\alpha}(\cdot|\gF_{t-1})$\;
        \State Compute an optimal action $a_t= \arg \max_{a\in \sA} g(a^\top \theta_t)$ \;
        \State Observe $r_t $ from $p_0(\cdot|\gF_{t-1}, a_t)$\;
        \State Update $X_t=X_{t-1} \cup \{ a_t,r_t\}$.\;
    \EndFor
\end{algorithmic}
\end{algorithm}

It is evident from Algorithm~\ref{alg:TS} that the true likelihood of generating data in~$\alpha$-TS is 
\(    p_{0}^{(T)}(X_T) := \prod_{t=1}^{T} \left[p_{0}(r_t|a_t)q_{\alpha}(a_t|\gF_{t-1})\right]\),
and we denote the corresponding data-generating distribution as $P_{0}^{(T)}$ till time $T$.
The expectation of the regret is taken with respect to $P_0^{(T-1)}$. For any $\theta \in \Theta$, we also denote the reward generating distribution given a sequence of actions $a^{(T)}:=\{a_1,a_2,\ldots,a_T\}$ as $P_{\theta,a}^{(T)}$ and the corresponding density as
\(    p_{\theta}^{(T)}(r^{(T)}|a^{(T)}) := \prod_{t=1}^{T} p_{\theta}(r_t|a_t).\)
For any $\beta >0$, we denote the $\beta$-R\'enyi divergence between two reward densities $p_{\theta}^{(t)}(r^{(t)}|a^{(t)})$ and $p_{\vartheta}^{(t)}(r^{(t)}|a^{(t)})$ for any $\theta,\vartheta \in \Theta$ as
       \(     D_{\beta}^{(t)}(\theta,\vartheta) := \frac{1}{\beta-1} \log \int p_{\theta}^{(t)}(r^{(t)}|a^{(t)})^{\beta} p_{\vartheta}^{(t)}(r^{(t)}|a^{(t)})^{1-\beta} dr^{(t)}
            = \sum_{s=1}^{t}\frac{1}{\beta-1} \log \int p_{\theta}(r|a_s)^{\beta} p_{\vartheta}(r|a_s)^{1-\beta} dr
            =: \sum_{s=1}^{t} D_{\beta}^{s}(\theta,\vartheta),\)
where $r^{(t)}:= \{r_1,r_2,\ldots,r_t\}$ and the second equality is due to the additivity property of $\alpha$-R\'enyi divergence~\citep[Theorem 28]{van2014renyi}.

\nomenclature[B]{$\Pi$}{Prior distribution over $\Theta$}
\nomenclature[B]{$(\Theta,\mathcal{Q})$}{Measurable parameter space}
\nomenclature[B]{$\Pi_{t,\alpha}$}{Fractional ($\alpha$-) posterior at time $t$}
\nomenclature[B]{$\alpha$}{Tempering parameter in $(0,1)$ controlling posterior variance}

\nomenclature[A]{$\E[\cdot]$}{Expectation with respect to the data-generating distribution}
\nomenclature[A]{$p_0^{(T)}(X_T)$}{Joint data-generating density under true parameter}
\nomenclature[A]{$P_0^{(T)}$}{Data-generating distribution up to time $T$}
\nomenclature[A]{$a^{(T)}$}{Sequence of actions $\{a_1,\dots,a_T\}$}
\nomenclature[A]{$r^{(T)}$}{Sequence of rewards $\{r_1,\dots,r_T\}$}
\nomenclature[A]{$P_{\theta,a}^{(T)}$}{Reward distribution under parameter $\theta$ and action sequence $a^{(T)}$}
\nomenclature[A]{$p_{\theta}^{(T)}(r^{(T)}|a^{(T)})$}{Joint reward density under $\theta$}
\nomenclature[A]{$D_{\beta}^{(t)}(\theta,\vartheta)$}{Rényi divergence of order $\beta$ between reward models up to time $t$}

%% file: Assumptions.tex
\section{Assumptions}
In this section, we provide the list of regularity conditions that we impose on the action space, prior, and the reward model to derive our general regret bound.
We first assume that the action space is compact, which is a standard assumption in the linear bandit literature~\citep {agrawal2013thompsonLIN,abeille2017linear,Dani2008}. 
\begin{assumption}[Action space]~\label{ass:action}
    We assume the $\sA$ is a closed and bounded subset of $\R^d$. In particular, without loss of generality, we assume that for any $a\in \sA$, $\|a\|\leq 1$, where $\|\cdot\|$ is the Euclidean norm. 
\end{assumption}
We also use $\|\cdot\|$ to denote the operator norm if used around a matrix.
We prove our regret bounds under some mild regularity conditions on the reward and prior distributions. 
The first assumption is on the link function $g(\cdot)$.
\begin{assumption}[Link function]~\label{ass:Link}
    The link function is Lipschitz continuous and strictly monotonic, that is, for any $x,y\in \R$, there exists a constant $C_g>0$, such that
             $   |g(x)-g(y)|\leq C_g|x-y|$,
    and there exists an $m>0$ such that $g'(x)> m$.
\end{assumption}

In many works on LB, it is common to assume that the rewards are generated from the following linear model: $
    r_t = a_t^\top \theta_0 + \eta_t$, where $\eta_t$ is a zero mean sub-Gaussian error. Note that for such models, the link function is identity. Moreover, for exponential family reward distributions (defined in~Definition~\ref{ass:rewardExp} formally), the link function is a standard terminology that defines the mean of the respective distribution in terms of its parameters.

Next, we impose another technical condition on the reward distribution that can be shown to be easily satisfied by both sub-Gaussian (Lemma~\ref{lem:SG}) and exponential (Lemma~\ref{lem:Exp}) family distributions.  
\begin{assumption}[Renyi divergence]~\label{ass:Renyi}
    At any time $t$, given $a^{(t)}$, we assume that for any $\theta,\vartheta \in \Theta$ and for some constant $C>0$, 
     \(   |a_t^\top\theta- a_t^\top \vartheta| \leq C\sqrt{\frac{2}{\alpha} D_{\alpha}^t(\theta,\vartheta)}.\)
\end{assumption}

\paragraph{Regularity conditions for $\alpha$-posterior contraction.}

The following regularity condition is a joint condition on the prior and the reward-generating distribution, which is crucial for obtaining near optimal minimax rate of convergence of the $\alpha$-posterior distribution.
        \begin{assumption}[Prior thickness]~\label{ass:prior}
            We assume that for a fixed $\alpha\in(0,1)$ and a positive sequence $\{\e_t\}_{t\geq0}$, there exist a $t_0\geq 1$, such that for all $t\geq t_0 $, $t\e_t^2\geq 2.4/\alpha$ and 
                \(\Pi(B\left(\theta_0,\e_t \right) ) \geq   e^{- \frac{D \alpha t \e_t^2 }{4} }\),
            where $B\left(\theta_0,\e_t \right) := \left\{  \theta \in \Theta:  D_2^{(t)}(\theta_0,\theta)\leq   \frac{D \alpha t \e_t^2 }{4}\right\} $ is a neighborhood of $\theta_0$ defined for any given $a^{(t)}$ and a positive constant $D$.
       \end{assumption}
The above assumption requires that, for any action sequence, the prior distribution places an exponentially decaying mass to a shrinking neighborhood of the true model parameter $\theta_0$. The neighborhood $B\left(\theta_0,\e_t \right)$ of $\theta_0$ is defined using a 2-R\'enyi divergence measure $D_2^{(t)}(\theta_0,\theta)$. Such type of assumptions are common in the works studying non-asymptotic convergence rate of the Bayesian posteriors~\citep{GGV,ZG,bhattacharya2019bayesian} and are satisfied by a large class of prior-likelihood combination for both parametric and non-parametric problems. Typically, $\e_t \to 0$ as $t\to\infty$ and determines the rate of convergence of the $\alpha$-posterior distribution.  We will later see in~Lemma~\ref{lem:post} that this assumption enables us to construct (near) minimax optimal rate of convergence of the $\alpha$-posterior distribution that is required for computing the bounds on the regret of $\alpha$-TS. In particular, observe that if $\e_t$ satisfies Assumption~\ref{ass:prior}, then any $\e'_t\geq \e_t$ satisfies it, and thus $\e_t$ is optimal. A common recipe that we follow to show this assumption is to locally bound the 2-R\'enyi divergence term by $t\|\theta-\theta_0\|^2$ (assuming $\Theta$ is an Euclidean space) and then appropriately control the rate at which prior mass of the shrinking 2-norm neighborhood decays. We show that the assumption above is satisfied by any bounded prior density with both exponential and sub-Gaussian families in Lemmas~\ref{lem:Exp_prior} and~\ref{lem:SG_prior}, respectively. This assumption is different from a similar assumption in~\cite{jaiswal2022}. Their assumption depends on the instance gap, which does not hold for an arbitrarily small instance gap.
We also impose another regularity condition on the prior density.
\begin{assumption}[sub-Gaussian prior]\label{ass:priorSG}
    \hspace{-1em} We assume that $\int_{\Theta} \Pi(d\theta) e^{\frac{\gamma}{2} \|\theta\|^2}< C_{\pi}$ for some $\gamma \leq 1$.
\end{assumption}
The assumption above can be easily satisfied by any sub-Gaussian prior distribution.

\paragraph{Regularity conditions for the finite sample Bernstein-von Mises for $\alpha$-posterior distribution.}~\label{par:FBvM}
In this section, we specify the regularity conditions developed in~\cite{Spokoiny2012,Panov2015} for finite sample Bernstein-von Mises (BvM) theorem to hold for parametric models. 
Given $a^{(t)}$, we denote the log-likelihood of the observed rewards $r^{(t)}$ as $\sL(\theta) = \sum_{s=1}^{t}\log p_{\theta}(r_s|a_s)$, where the dependence on $r^{(t)}$ is implicit in the notation.

$\nabla \sL(\theta)$ denotes the gradient of $\sL(\cdot)$ evaluated at $\theta$ and $\nabla^2\E_a\sL(\theta)$ stands for the Hessian of the expected log-likelihood, where $\E_a[\cdot]:=\E[\cdot|a^{(t)}]$ is the expectation with respect to $P_{0,a}^{(t)}$, conditioned on the action sequence $a^{(t)}$. Define 
   \( \sD_0^{2}
     := 
    - \nabla^{2} \E_a \sL(\theta_0) .\)
Note that $\sD_0^{2}$ 
is defined akin to the Fisher information matrix of $P_{\theta,a}^{(t)}$ at $\theta_0$. Also define $\sD_0^{2}(\theta)
     :=
    - \nabla^{2} \E_a \sL(\theta) $ and note that $\sD_0^{2}= \sD_0^{2}(\theta_0)$. We also denote by $\sD_0$ the positive definite square root of $\sD_0^2$.
Denote the maximum likelihood estimator as $\hat{\theta}_t := \arg\max_{\theta \in \Theta} \sL(\theta)$ and $\theta_0 = \arg\max_{\theta \in \Theta} \E_a[\sL(\theta)]$. The stochastic part of the log-likelihood is denoted as $\zeta (\theta):= \sL(\theta) - \E[\sL(\theta)]$.
 The regularity conditions required for the finite sample BvM to hold are divided into local and global. The local conditions only describe the properties of \(\sL(\theta)\)
  for \(\theta \in \Theta(\BFr_0) \) with some fixed value \(\BFr_0\), where
   \( \Theta_0(\BFr_0)
    :=
    \big\{\theta \in \Theta \colon \|\sD_0(\theta - \theta_0)\| \le  \BFr_0 \big\}.\)
The global conditions have to be fulfilled on the whole \(\Theta\).  These conditions are constructed to replace the celebrated \textit{local asymptotic normality} (LAN) condition~\citep[Chapter 7]{van2000asymptotic} on the $\sL(\theta)$, which is required for the asymptotic Bernstein-von Mises theorem. Recall that LAN is essentially a local quadratic approximation of $\sL(\theta)$ in the vicinity of $\theta_0$.

\begin{assumption}~\label{ass:Spokoiny}
We start with some exponential moment conditions. 
  \begin{enumerate}
    \item[\({(E\!D_{0})} \)]
      There exists a constant \(\nu_0>0\), 
       and a constant \( \BFg > 0 \) 
      such that
      \(  \sup_{\gamma \in \R^{d}} \log\E_a \exp\left\{
              \BFm \frac{\langle \nabla \zeta(\theta_0),\gamma \rangle}
              {\| \sD_0 \gamma \|}
              \right\}
        \le
        \frac{\nu_0^{2} \BFm^{2}}{2}, ~ \forall \BFm: |\BFm| \le \BFg.\)
    \item[\( {(E\!D_{1})} \)]
      There are constants \( \omega > 0 \) and 
      for each $\BFr>0$ 
      a constant \(\BFg(\BFr) > 0\) such that for
      all \( \theta \in \Theta_0(\BFr) \):
      \(  \sup_{\gamma_1,\gamma_2 \in \R^{d}} \sup_{\theta \in \Theta_0(\BFu)} \log \E_a \exp\left\{
        \frac{\BFm}{\omega} \frac{\gamma_1^{\top}\nabla^2 \zeta(\theta)\gamma_2 }{ \|\sD_0 \gamma_1\|\|\sD_0 \gamma_2\|}
        \right\}
        \le
        \frac{\nu_0^{2} \BFm^{2}}{2}, ~ \forall \BFm: |\BFm| \le \BFg(\BFr).\)
  \end{enumerate}
  
  \begin{enumerate}
    \item[\({(\gL_0)} \)]
      There exists a constant \(\delta(\BFr)\) such that for all $\theta\in\Theta_0(\BFr)$ and for all \(\BFr \le \BFr_0\),
      \( \big\|\sD_0^{-1} \sD_0^2(\theta)\sD_0^{-1} - I_{d} \big\|
        \le
        \delta(\BFr).\)
  \end{enumerate}
 
  The following is the required global condition.
  \begin{enumerate}
    \item[\( {(\gL{\BFr})} \)]
      For any \(\BFr>0\) there exists a value \(\BFb(\BFr) > 0\),
      such that  $\BFr b(\BFr) \to \infty$ as $\BFr \to \infty$ and 
       \(  -\E \sL(\theta,\theta_0)
         \ge
         \BFr^{2} b(\BFr) \quad \text{for all \( \theta \) with } 
         \BFr = \|\sD_0 (\theta - \theta_0)\|.\)
  \end{enumerate}
\end{assumption}
 At a very high level, the conditions $(ED_0)$ and $(ED_1)$ are the exponential moment conditions that essentially require the tails of the reward distribution to be exponentially decaying. Moreover, $(\gL_0)$ imposes a local identifiablity condition, that ensures that $-\E_a[\sL(\theta)-\sL(\theta_0)]$ are bounded above and below by a quadratic function of $\|\theta-\theta_0\|$.  $(\gL\BFr)$ is a global identifiability condition that requires the gap $-\E_a[\sL(\theta)-\sL(\theta_0)]$ grows in a controlled fashion as $\BFr= \|\sD_0(\theta-\theta_0)\|$ increases. More, discussion regarding these assumptions can be found in Section 2 of~\cite{Spokoiny2012} and partly in~\cite{Panov2015}.

 Note that Assumption~\ref{ass:prior}, which is needed to derive first-order concentration properties of the $\alpha$-posterior distribution, effectively requires no regularity condition on the reward distribution.
 However, to conduct a more refined second-order BvM-type analysis, a more comprehensive set of local and global moment and identifiability conditions is required. In Section~\ref{sec:Spokoiny}, we show that the exponential (Definition~\ref{ass:rewardExp}) and sub-Gaussian  (Definition~\ref{ass:reward} with some mild smoothness conditions on the density of the error distribution) families satisfy all the conditions in Assumption~\ref{ass:Spokoiny}. More examples can be found in~\cite{Spokoiny2012,Panov2015}.
 
 To extend the lower bound computed in~\cite[Theorem 4]{Panov2015} to general priors, we need to control the prior density on the set $\Theta_0(\BFr_0)$. In particular, observe that
 \(   \Pi_{t,\alpha}(\Theta_0(\BFr_0)|F_{t})  
    =
    \frac{\int_{\Theta} \exp \big\{ {\alpha} \sL(\theta, \theta_0) \big\} \Pi(d\theta)
            \1\big\{\theta \in \Theta_0(\BFr_0)\big\} }
         {\int_{\Theta} \exp \big\{ {\alpha} \sL(\theta, \theta_0) \big\}
            \big\} \Pi (d\theta) }.\)
Note, that $\Theta_0(\BFr_0)= \big\{\theta \in \Theta \colon \|\sD_0(\theta - \theta_0)\| \le  \BFr_0 \big\} 
$. So for a fixed $\BFr_0$, the prior density ($\pi(\theta):=\Pi(d\theta)$) can easily be lower bounded on $\Theta_0(\BFr_0)$ by fixed number that depends on $\BFr_0$ and $\theta_0$. In fact, if $\BFr_0 =  \frac{c}{\sqrt{t}}$, for some constant $c>0$, then note that  $\min_{\theta\in \Theta_0(\BFr_0) } \pi(\theta)\geq \min_{\theta\in \Theta_0(c) } \pi(\theta)$, because, $\Theta_0(\BFr_0) \subseteq \Theta_0(c)$ for all $t\geq 1$. Therefore,

\begin{assumption}[Prior]
\label{ass:Prior2}
We make following assumption on the prior distribution.
    \begin{enumerate}
        \item The prior density is continuous on $\Theta$.
        \item There exists a positive constant $\BFM$, such that the prior density, $\pi(\theta)\leq \BFM$ for all $\theta \in \Theta$.
        \item For any compact set $K\subset \Theta$, $\pi(\theta)>0$ for all $\theta\in K$.
    \end{enumerate}
\end{assumption}

As our focus is solely on the lower bound outcome described in~\cite{Panov2015}, we observe that, under the stated assumption, the outcome in Theorem~4 of~\cite{Panov2015} readily follows with a factor denoted by $\frac{\min_{\theta\in \Theta_0(c) } \pi(\theta)}{\BFM} := \mathbb{C}$. For the parametric problems of interest in this study, the two conditions (stated in Assumptions~\ref{ass:prior} and~\ref{ass:Prior2}) imposed on the prior distribution are quite lenient, merely necessitating the prior density to be positive, continuous, and bounded. However, Assumption~\ref{ass:priorSG} requires the prior distribution to have sub-Gaussian tails.

\nomenclature[R]{$\mid \mid \cdot\mid 
\mid$}{Euclidean norm (vector) or operator norm (matrix)}
\nomenclature[R]{$C_g$}{Lipschitz constant of the link function}
\nomenclature[R]{$m$}{Lower bound on derivative of the link function}
\nomenclature[R]{$C$}{Constant in the Rényi divergence condition}
\nomenclature[R]{$\e_t$}{Posterior contraction rate at time $t$}
\nomenclature[R]{$B(\theta_0,\e_t)$}{Neighborhood of $\theta_0$ defined via Rényi divergence}
\nomenclature[R]{$D$}{Constant in prior mass condition}
\nomenclature[R]{$C_\pi$}{Constant controlling prior tail decay}
\nomenclature[R]{$\gamma$}{Parameter in sub-Gaussian prior condition}

\nomenclature[R]{$\sL(\theta)$}{Log-likelihood of observed rewards $r^{(t)}$ given actions $a^{(t)}$, with data dependence implicit}
\nomenclature[R]{$\nabla \sL(\theta)$}{Gradient of log-likelihood}
\nomenclature[R]{$\nabla^2 \sL(\theta)$}{Hessian of log-likelihood}
\nomenclature[R]{$\E_a[\cdot]$}{Expectation w.r.t. $P_{0,a}^{(t)}$ (given action sequence)}
\nomenclature[R]{$\sD_0^2$}{Fisher information matrix at $\theta_0$}
\nomenclature[R]{$\sD_0^2(\theta)$}{Expected negative Hessian at $\theta$}
\nomenclature[R]{$\hat{\theta}_t$}{Maximum likelihood estimator}
\nomenclature[R]{$\zeta(\theta)$}{Centered log-likelihood (stochastic component)}

\nomenclature[R]{$\Theta_0(\BFr_0)$}{Local neighborhood around $\theta_0$}
\nomenclature[R]{$\BFr_0$}{Radius of local neighborhood}
\nomenclature[R]{$\BFr$}{Generic radius parameter}
\nomenclature[R]{$\delta(\BFr)$}{Local curvature deviation function}
\nomenclature[R]{$\BFb(\BFr)$}{Growth rate in global identifiability condition}

\nomenclature[R]{$\nu_0$}{Variance proxy in exponential moment condition}
\nomenclature[R]{$\BFg$}{Range parameter for exponential moment condition}
\nomenclature[R]{$\omega$}{Scaling constant in Hessian moment condition}
\nomenclature[R]{$\BFg(\BFr)$}{Local exponential moment bound}

\nomenclature[R]{$\pi(\theta)$}{Prior density}
\nomenclature[R]{$\BFM$}{Upper bound on prior density}
\nomenclature[R]{$\mathbb{C}$}{Constant from prior density bounds}

%% file: LowerBound.tex
\section{Lower Bound  for $\alpha$-TS}We present a lower bound construction illustrating how the regret of $\alpha$-TS depends on the dimension $d$ and the tempering parameter $\alpha$. The construction extends the classical hypercube instance used in linear bandit lower bounds. 

\begin{proposition}~\label{prop:lowerbound}
Consider the linear bandit model with action set 
\(
\sA = [-1,1]^d,
\)
and rewards 
\(
r_t = a_t^\top \theta_0 + \eta_t,\quad \eta_t \sim \mathcal N(0,1).
\)
Let the parameter satisfy \(\theta_0 \in \{\pm \mu\}^d.\) for some $\mu>0$
Then the regret of $\alpha$-TS with standard Gaussian prior satisfies
\[
\mathbb E[R_T]
\;\ge\;
2\mu d\sum_{t=1}^T
\mathbb E\!\left[
\Phi\!\left(
-\frac{|\mu_{t,i}^{(\alpha)}|}
{\sqrt{((B_t^\alpha)^{-1})_{ii}}}
\right)
\right],
\]
where $\mu_t^\alpha$ and $B_t^\alpha$ denote the $\alpha$-posterior mean and precision matrix. Moreover, the signal-to-noise ratio satisfies
\(
\frac{|\mu_{t,i}^{(\alpha)}|}
{\sqrt{((B_t^\alpha)^{-1})_{ii}}}
\;\le\;
\mu\sqrt{1+\alpha t}
+
\mu\sqrt{d}
+
\sqrt{\alpha}|Z_{t,i}|,
\)
where $Z_{t,i}$  is a conditionally centered Gaussian random variable with variance at most $1$.
\end{proposition}

The proof of the above proposition is provided in Appendix~\ref{app:Proof_of_lowerbound}. Proposition~\ref{prop:lowerbound} characterizes the quantity controlling the probability of a coordinate sign error under $\alpha$-TS. In particular, the error probability is governed by the signal-to-noise ratio
\(
\frac{|\mu_{t,i}^{(\alpha)}|}
{\sqrt{((B_t^\alpha)^{-1})_{ii}}},
\)
which determines the probability that the sampled parameter disagrees with the sign of the true coordinate. Using the bounds derived in the proof, this ratio satisfies
\(
\frac{|\mu_{t,i}^{(\alpha)}|}
{\sqrt{((B_t^\alpha)^{-1})_{ii}}}
\;\le\;
\mu\sqrt{1+\alpha t}
+
\mu\sqrt{d}
+
\sqrt{\alpha}\,|Z_{t,i}|,
\)
where $Z_{t,i}$ is a conditionally centered Gaussian random variable with variance at most $1$. This expression reveals the key interaction between the signal magnitude $\mu$, the dimension $d$, and the tempering parameter $\alpha$. 
The dominant deterministic component grows as $\mu\sqrt{1+\alpha t}$, which determines the rate at which the posterior concentrates around the true parameter.

To make this dependence explicit, consider the parameter choice
\(
\mu = c\sqrt{\frac{d}{T}}
\)
for a constant $c>0$.
From the first paragraph in the proof of Proposition~\ref{prop:lowerbound}, recall that the coordinate error probability can be lower bounded as
\(
\mathbb P(a_{t,i}\neq \operatorname{sign}(\theta_{0,i}))
\;\ge\;
\mathbb E\!\left[
\Phi\!\left(
-\frac{|\mu_{t,i}^{(\alpha)}|}{\sqrt{((B_t^\alpha)^{-1})_{ii}}}
\right)
\right].
\)
Let \(E_{t,i} = \{|Z_{t,i}|\le u\}\). Conditioning on this event and using the law of total expectation, we obtain
\begin{align*}
\mathbb P(a_{t,i}\neq \operatorname{sign}(\theta_{0,i}))
&\ge
\mathbb E\!\left[
\Phi\!\left(
-\frac{|\mu_{t,i}^{(\alpha)}|}{\sqrt{((B_t^\alpha)^{-1})_{ii}}}
\right)
\mathbf{1}\{E_{t,i}\}
\right] \\
&=
\mathbb P(E_{t,i})
\,
\mathbb E\!\left[
\Phi\!\left(
-\frac{|\mu_{t,i}^{(\alpha)}|}{\sqrt{((B_t^\alpha)^{-1})_{ii}}}
\right)
\;\middle|\; E_{t,i}
\right].
\end{align*}
On the event \(E_{t,i}\), we have the deterministic upper bound
\(
\frac{|\mu_{t,i}^{(\alpha)}|}{\sqrt{((B_t^\alpha)^{-1})_{ii}}}
\;\le\;
c\sqrt{\frac{d}{T}+\alpha d}
+
c\frac{d}{\sqrt{T}}
+
\sqrt{\alpha}\,u.
\)
Since the Gaussian CDF \(\Phi(-x)\) is decreasing in \(x\), this implies that, on \(E_{t,i}\),
\(
\Phi\!\left(
-\frac{|\mu_{t,i}^{(\alpha)}|}{\sqrt{((B_t^\alpha)^{-1})_{ii}}}
\right)
\;\ge\;
\Phi\!\left(
-
c\sqrt{\frac{d}{T}+\alpha d}
-
c\frac{d}{\sqrt{T}}
-
\sqrt{\alpha}\,u
\right).
\)
Taking expectation conditional on \(E_{t,i}\) yields the same lower bound, and therefore
\[
\mathbb P(a_{t,i}\neq \operatorname{sign}(\theta_{0,i}))
\;\ge\;
\mathbb P(|G|\le u)
\,
\Phi\!\left(
-
c\sqrt{\frac{d}{T}+\alpha d}
-
c\frac{d}{\sqrt{T}}
-
\sqrt{\alpha}\,u
\right),
\]
where \(G\sim \mathcal N(0,1)\).
For sufficiently large $T$, the second term in the argument of $\Phi$ becomes negligible, and the coordinate error probability is therefore bounded below by a constant multiple of $\Phi(-c\sqrt{\alpha d})$. Substituting this bound into the regret decomposition yields
\(
\mathbb E[R_T]
\;\ge\;
d\mu T\,
\mathbb P(|G|\le u)
\,
\Phi\!\left(
-
c\sqrt{\alpha d}
-
c\frac{d}{\sqrt{T}}
-
\sqrt{\alpha}\,u
\right).
\)
Using $\mu=c\sqrt{d/T}$, we obtain
\[
\mathbb E[R_T]
\;\ge\;
c\,d^{3/2}\sqrt{T}\,
\mathbb P(|G|\le u)
\,
\Phi\!\left(
-
c\sqrt{\alpha d}
-
c\frac{d}{\sqrt{T}}
-
\sqrt{\alpha}\,u
\right).
\]
This expression makes the dimension--temperature interaction explicit. 
In particular, when $\alpha d$ remains bounded by a constant and $T$ is sufficiently large, the Gaussian tail term remains bounded away from zero, and the regret scales as
\(
\mathbb E[R_T] = \Omega(d^{3/2}\sqrt{T}).
\)
Thus the $d^{3/2}$ dependence appearing in the upper bound cannot be avoided for this class of posterior-sampling algorithms. We formally write the outcome of the discussion above in the following proposition (without proof).

\begin{proposition}[Dimension--temperature lower bound]
\label{prop:alphad_lower}
Suppose the conditions of Proposition~\ref{prop:lowerbound} hold and consider the parameter choice \(
\mu = c\sqrt{\frac{d}{T}}
\)
for a constant $c>0$. Then for any $u>0$ 
the expected regret satisfies
\[
\mathbb E[R_T]
\;\ge\;
c\,d^{3/2}\sqrt{T}\,
\mathbb P(|G|\le u)
\,
\Phi\!\left(
-
c\sqrt{\frac{d}{T}+\alpha d}
-
c\frac{d}{\sqrt{T}}
-
\sqrt{\alpha}\,u
\right).
\]
In particular, if $\alpha d = O(1)$ and $d=o(\sqrt{T})$, then
\(
\mathbb E[R_T]
=
\Omega(d^{3/2}\sqrt{T}).
\)
\end{proposition}

More generally, the bound shows that the regret constant depends on the product $\alpha d$ through the factor $\Phi(-c\sqrt{\alpha d})$, illustrating how the tempering parameter controls the effective signal-to-noise ratio of posterior sampling. Intuitively, smaller values of $\alpha$ inflate the posterior covariance and increase the likelihood of sampling parameters whose signs disagree with the true parameter, while larger values of $\alpha$ lead to faster posterior concentration and reduce the probability of coordinate errors. The above construction therefore highlights the fundamental role of the scaling of $\alpha$ with the dimension $d$ in determining the regret behavior of $\alpha$-TS.

\nomenclature[D]{$\mu_t^\alpha$}{\(\alpha\)-posterior mean at time \(t\) in lower bound construction}
\nomenclature[D]{$B_t^\alpha$}{\(\alpha\)-posterior precision matrix in lower bound construction}

%% file: Analysis.tex
\section{Frequentist Regret Bound}

In this section, we present our main result on bounding the regret of $\alpha$-TS. The general bound on the expected regret of $\alpha$-TS is provided below. 

\begin{theorem}[General regret bound]~\label{thm:GRB}
    Under Assumptions~\ref{ass:Link},~\ref{ass:Renyi},~\ref{ass:prior},~\ref{ass:priorSG},~\ref{ass:Spokoiny}, and~\ref{ass:Prior2}, we show for any $\eta>0$ and $\alpha(1-\alpha)\lambda<1$ that
    \begin{align}
    \nonumber
        \E[\mReg(T)] &\leq {C_g} \left( 2d \log \left(1+\lambda^{-1} {T}\right)\sum_{t=1}^{T} t \e_t^2 \right)^{1/2} \Big[ \left(\bar p_0(d) (1-2e^{-\eta}) \right)^{-1/2} 
       \\
    &+   \left(   (1-\alpha)^{-1}(D+2)C(\alpha,\lambda,\theta_0)  \right)^{1/2} \left(1+ \left( \bar p_0(d) (1-2e^{-\eta})  \right)^{-1/2}\right)
    \Big].
    \end{align}  
    where $\bar p_0(d) = \mathbb{C}\exp\Big(-2\alpha\Delta_t(d,\eta) $ 
$ - 8e^{-\eta}\Big)  \left\{ \frac{C\sqrt{\alpha (d+\eta)}}{1+C^2{\alpha (d+\eta)}}e^{-C^2{\alpha (d+\eta)}/2}  \right\}-e^{-\eta} -C(\alpha,\lambda,\theta_0)2 e^{ - \frac{\alpha}{4}  t_0\e_{t_0}^2 }$ for some  $\Delta_t(d,\eta)$ and $t_0$, and $C(\alpha,\lambda,\theta_0)$ grows with $\alpha(1-\alpha)$ and also depends on $C_{\pi}$. $\Delta_t(d,\eta)$ is a known non-random function of $\eta$,$d$, and $t$ and it decreases as $t$ increases.
\end{theorem}

The upper bound on the regret above depends on the time horizon $T$ and the dimension of the action space $d$, among other terms defined in the regularity conditions. In many parametric modeling instances, the term $\e_t^2$ is $O\left(\frac{d}{\alpha}\frac{\log t}{t}\right)$. Observe that $d\sum_{t=1}^{T} t\e_t^2= O(\frac{d^2}{\alpha} \sum_{t=1}^{T} \log t )=O(\frac{d^2}{\alpha}T\log T)$, which could provide the desired upper bound that matches the lower bound in~\cite{Dani2008} for linear bandit problems with respect to both $d$ and $T$. However, note that the term $\bar p_0(d)$ also depends on $d$, and for any $\alpha\in(0,1)$, it will result in a regret bound that grows exponentially with $d$. To balance this, we must choose a $d$-dependent $\alpha$. 
Since, $\Delta_t(d,\eta)$ decreases as $t$ increases, for $\eta=O(d)$, we can choose sufficiently large warm-up iterations $t_1=\Omega(d)$ to ensure that $\Delta_t(d,\eta)\leq 1$. Consequently, for both exponential (Corollary~\ref{corr:GRB_Exp}) and sub-Gaussian (Corollary~\ref{corr:GRB_SG}) families, we choose $\alpha=d^{-1}$ and obtain the state-of-the-art regret bounds of $O(d^{3/2}\sqrt{T}\log T)$. 
For this choice of $\alpha$, other $\alpha$-dependent terms such as $C(\alpha,\lambda,\theta_0)$ and $(1-\alpha)^{-1}$ do not grow with $d$ and are bounded away by $0$. Note that $\alpha=d^{-1}$ is an optimal choice of $\alpha$ because to nullify the exponential effect due to $\alpha(d+\eta)$, we must choose $\alpha=d^{-(1+\varepsilon)}$ for any $\varepsilon\geq 0$ which would results into an upper bound of $O(d^{3/2+\varepsilon/2}\sqrt{T}\log T)$. 

As noted in the introduction, our analysis adapts the idea of constructing saturated and unsaturated sets of action introduced by~\cite{agrawal2013thompsonLIN} and combines it with the first and second-order properties of the $\alpha$-posterior distribution. To present the proof sketch of the theorem above, we state the following two definitions. 
\begin{definition}[Precision matrix]~\label{def:Precision}
    Define $V_t = \lambda I + \sum_{s=1}^{t-1} a_s a_s^\top$ for any $\lambda>0$, and $\|x\|_{A} := \sqrt{x^\top A x}$ is the matrix norm of $x$ for any square matrix $A$.
\end{definition}
This is a standard definition of the precision matrix that is derived while solving a Bayesian linear regression problem with Gaussian prior (with covariance $\lambda I$) and standard Gaussian error. Next, we define the set of saturated actions.

 \begin{definition}[Saturated actions]~\label{def:US}
            Any action $a\in \sA$ is called saturated at time $t$ or $a\in \mathcal{C}_t$, where $\mathcal{C}_t$ is the set of saturated actions, if $g(a_0^\top\theta_0) -   g(a^\top \theta_0) \geq {C_g}\sqrt{t} \e_t \|a\|_{V_t^{-1}}$ and otherwise it is called unsaturated. 
        \end{definition}
This definition is adapted from~\cite{agrawal2013thompsonLIN}. Here, $\e_t$ is the same as defined in Assumption~\ref{ass:prior}, which characterizes the $\alpha$-posterior contraction rate. Note that at any time $t$, the best action $a_0$ is always unsaturated. Next, we provide a proof sketch of our main result, which also highlights the need for computing first and second-order properties of the $\alpha$-posterior. We would also like to remark that to derive regret bounds in expectation, we use a separate set of techniques than those used in~\cite{agrawal2013thompsonLIN} and~\cite{abeille2017linear}. 
\paragraph{Proof sketch of Theorem~\ref{thm:GRB}.}

    Recall the definition $\mReg(T)$ and let $\bar a_t:= \arg\min_{a \notin \mathcal{C}_t} \|a\|_{V_t^{-1}} $. Now observe that 
     \vspace{-0.7em}
    \begin{align}
         g(a_0^\top\theta_0) -   g(a_t^\top\theta_0) = \underbrace{ g(a_0^\top\theta_0) -    g(\bar a_t^\top\theta_0) }_{\textsc{(I)}} + \underbrace{ g(\bar a_t^\top\theta_0) -   g(a_t^\top \theta_t) }_{\textsc{(II)}} + \underbrace{g(a_t^\top\theta_t) -   g(a_t^\top \theta_0 )}_{\textsc{(III)}}.
        \label{eq:GBps0}
    \end{align}
    \vspace{-0.1em}
    Using Definition~\ref{def:US}, observe that the expected cumulative sum of \textsc{(I)} can be bounded above by
         \({C_g} \left(\sum_{t=1}^{T} t \e_t^2 \right)^{1/2} \left(\sum_{t=1}^{T} \E[\|\bar a_t\|_{V_t^{-1}}^2] \right)^{1/2}.\)
    The expected cumulative sum of \textsc{(II)} can be bounded by \(C_g \left( \E \left[ \sum_{t=1}^T \|\bar a_t\|_{V_t^{-1}}^2\right]  \right)^{1/2}\left(   \sum_{t=1}^T  \E \left[\|\theta_t -\theta_0 \|_{V_t}^2 \right] \right)^{1/2}\) using Assumption~\ref{ass:Link} and the fact that $g(a_t^\top \theta_t) > g(\bar a_t^\top \theta_t)$.
   Similarly, the expected cumulative sum of \textsc{(III)} can be bounded above by \(C_g\left( \E \left[ \sum_{t=1}^T \|a_t\|_{V_t^{-1}}^2\right]  \right)^{1/2}\) ~ \( \left(   \sum_{t=1}^T  \E \left[\|\theta_t -\theta_0 \|_{V_t}^2 \right] \right)^{1/2}.\)

To further simplify the bound observe that the term $\left( \E \left[ \sum_{t=1}^T \|a_t\|_{V_t^{-1}}^2\right]  \right)^{1/2}$ can be bounded by $\sqrt{ 2d \log \left(1+\lambda^{-1} {T}\right)}$ using elliptical potential lemma from~\cite[Lemma 3.1]{Auer2002} and~\cite[Proposition 1]{carpentier_elliptical_2020}. Observe that  quantifying a bound on~$\left(   \sum_{t=1}^T  \E \left[\|\theta_t -\theta_0 \|_{V_t}^2 \right] \right)^{1/2}$ requires understanding the first order concentration properties of the $\alpha$-posterior, that we derive in~Lemma~\ref{lem:post}.
Now, the only term that remains to be analyzed is $\E \left[ \|\bar a_t\|_{V_t^{-1}}^2\right]$. Observe that 
   \( \E \left[ \|a_t\|_{V_t^{-1}}^2|\gF_{t-1}\right] \geq \E \left[ \|a_t\|_{V_t^{-1}}^2|\gF_{t-1}, a_t \notin \mathcal{C}_t\right] \Pi_{\alpha,t}(a_t \notin \mathcal{C}_t| \gF_{t-1} )
    \geq \|\bar a_t\|_{V_t^{-1}}^2 \Pi_{\alpha,t}(a_t \notin \mathcal{C}_t| \gF_{t-1} ),\)
where the second inequality uses the definition of $\bar a_t$ and the fact that it is completely determined by $\gF_{t-1}$. The remaining steps establish a lower bound on  $\Pi_{\alpha,t}(a_t \notin \mathcal{C}_t| \gF_{t-1} )$ to upper bound $\E \left[ \|\bar a_t\|_{V_t^{-1}}^2\right]$ by $\E \left[ \|a_t\|_{V_t^{-1}}^2\right]$ upto some time-dependent term and then use elliptical potential lemma. 
If $g(a_0^\top \theta_t) > g(a_j^\top \theta_t)$ for all saturated actions, i.e.$ \forall j\in \mathcal{C}_t$, then one of the unsaturated actions must be played because $a_0$ is always unsaturated. Using this fact, we establish a lower bound on the probability of selecting an unsaturated action as
\begin{align}
    \nonumber
    \Pi_{\alpha,t}&(a_t \notin \mathcal{C}_t| \gF_{t-1} ) \geq \Pi_{\alpha,t}( g(a_0^\top \theta_t) > g(a_j^\top \theta_t), \forall j\in \mathcal{C}_t | \gF_{t-1} ) 
    \\
     & \geq \Pi_{\alpha,t} (a_0^\top \theta_t > a_0^\top \theta_0  | \gF_{t-1} )  -  \Pi_{\alpha,t}(  \|\theta_t -\theta_0 \|_{V_t}\geq \sqrt{t} \e_t | \gF_{t-1} ).
    ~\label{eq:eqpsLB2}
    \end{align}
To establish a lower bound on $\Pi_{\alpha,t}(a_t \notin \mathcal{C}_t| \gF_{t-1} )$, we upper bound $ \Pi_{\alpha,t}(  \|\theta_t -\theta_0 \|_{V_t}\geq \sqrt{t} \e_t | \gF_{t-1} )$ by using $\alpha$-posterior contraction result in~Lemma~\ref{lem:post} and lower bound  $\Pi_{\alpha,t} (a_0^\top \theta_t > a_0^\top \theta_0  | \gF_{t-1} )$ using the lower bound of the finite sample BvM stated in~Lemma~\ref{lem:FBVM}. The detailed proof of the theorem and the associated lemmas are provided in the Appendix~\ref{app:Proof}.

\subsection{First and second-order properties of the $\alpha$-posterior}

In this section, we present two important results that characterize the first and second-order properties of the $\alpha$-posterior distribution using the regularity conditions on the prior and reward distribution in the previous section. 

\begin{lemma}[$\alpha$-posterior contraction]\label{lem:post}
        Fix $\alpha\in(0,1)$. Under Assumptions~\ref{ass:prior} and~\ref{ass:priorSG} for any given sequence of actions $a^{(t)}$ that is adapted to $\gF_t$ and $t>0$, we have $\forall j>0$ and $\alpha(1-\alpha)\lambda<1$ that,
        \begin{align}
            P_{0} \left(\Pi_{t,\alpha}\left(  \frac{\lambda \alpha}{2}\|\theta-\theta_0\|^2 + D_{\alpha}^{(t)}(\theta,\theta_0) \geq \frac{(D+j)\alpha}{(1-\alpha)2} t\e_t^2 |\gF_{t} \right)\right)  \leq 2 C(\alpha,\lambda,\theta_0) e^{ - \frac{j\alpha}{4}  t\e_t^2 }, \text{and}
            \\
            \E \left[E_{\alpha}\left[ \frac{\lambda \alpha}{2}\|\theta-\theta_0\|^2 +  D_{\alpha}^{(t)}(\theta,\theta_0) | \gF_{t} \right] \right]
            \leq  \frac{(D+2)\alpha}{(1-\alpha)2} t\e_t^2 + \frac{4C(\alpha,\lambda,\theta_0) e^{ - \frac{\alpha}{4}  t \e_t^2 } }{(1-\alpha)}.
        \end{align}
        
    \end{lemma}

The first result computes a bound on the posterior probability of $\theta$ being away from $\theta_0$. This bound computes the rate at which the $\alpha$-posterior shrinks. The second result quantifies a concentration bound in expectation. Intuitively, since, $D_{\alpha}^{(t)}(\theta,\theta_0)$ scales with $t$, $\e_t$ determines the rate at which the $\alpha$-posterior distribution concentrates to the truth.  The proof of both results is adapted from the posterior concentration analysis in the Bayesian statistics literature. 

\textit{Remark:}
Typical analysis of the $\alpha$-posterior concentration bounds measure deviation using only $D_{\alpha}^{(t)}(\theta,\theta_0)$ term. However, in the analysis of the regret, we needed to quantify a bound on the expectation of $\|\theta-\theta_0\|_{V_t}$, where $V_t$ is the precision matrix defined in~Definition~\ref{def:Precision}, therefore, we modify the analysis to incorporate this additional term $\lambda\|\theta-\theta_0\|$ due to $V_t$-norm. Also, it is only due to this modified definition of the discrepancy measure; we need an additional assumption (Assumption~\ref{ass:priorSG}) on the prior that requires its tails to be sub-Gaussian. Recall that $C_{\pi}$ is the bound on the exponential moment on the prior as defined in~Assumption~\ref{ass:priorSG}. 

Next, we present below the relevant part of the~\cite[Theorem 4]{Panov2015}.

\begin{lemma}[Finite BvM]~\label{lem:FBVM}
    Under Assumptions~$(E\!D_0),(E\!D_1),(\gL_0),(\gL\BFr)$ on the reward distribution, and Assumption~\ref{ass:Prior2} on the prior density, for any measurable set $A\subseteq \R^d$ and $\alpha\in(0,1)$,
    \begin{align}
        \Pi_{t,\alpha}(\sD_0(\theta- \bar \theta_t)\in A|\gF_{t}) \geq \mathbb{C}\exp\left(-2\alpha \Delta_t(d,\eta)  - 8e^{-\eta}\right) \left\{P(\phi_d/\sqrt{\alpha} \in A)  \right\} -e^{-\eta}, 
    \end{align}
    with $P_{0,a}^{(t)}$ probability of at least $1-e^{-\eta}$, where $\phi_d$ is a standard Gaussian random variable  in $\R^d$, $\bar \theta_t = \theta_0 + \sD_0^{-2} \nabla\sL(\theta_0)$, $\Delta_t(d,\eta)$ is a known non-random function of $\eta$,$d$, and $t$ and it decreases as $t$ increases. Moreover, $\|\sD_0(\theta_0-\bar \theta_t)\|\leq {{C^\dagger}\sqrt{{d+\eta}}}$ with $P_{0,a}^{(t)}$-probability of at least $1-e^{-\eta}$ for some positive constant ${C^\dagger }$.
\end{lemma}
The proof of the result above is a direct consequence of~\cite[Theorem 4]{Panov2015} and the result in the display (33) of~\cite[Theorem 9]{Panov2015}. Analogous to the classical BvM result~\cite[Section 10.2]{van2000asymptotic}, observe that $\sD_0(\theta- \bar \theta_t)$ perturbs parameter $\theta$ by $\bar \theta_t$, which is the first order approximation of the maximum likelihood estimate and scale it by $\sD_0$ which accounts for the Fisher information matrix. Intuitively, as $t\to\infty$, the result above recovers (lower limit of) the classical Bernstein-von Mises theorem. The factor $\alpha$ accounts for the effect of tempering the likelihood in the definition of the $\alpha$-posterior distribution. In particular, it can be observed from the lower bound above that $\alpha\in (0,1)$ inflates the posterior variance by a factor $\alpha^{-1}$. Recall, $\Delta_t(d,\eta)$ appears in the expression for $\bar p_0(d)$ in Theorem~\ref{thm:GRB}. We will show in the examples that $\Delta_t(d,\eta)\leq 1 
$, for any $t\geq t_1$. 

\section{Examples}~\label{sec:Spokoiny}
We consider two broad classes of reward distributions: the sub-Gaussian and the exponential families. Below, we clearly state the definitions and assumptions for these families. 

\subsection{Exponential family}

First, we specify the conditions for the exponential family reward models. 

    \begin{definition}[Exponential]~\label{ass:rewardExp}
        We assume for $\Theta \subseteq \R^d$ that 
        \begin{itemize}
            \item[i.] $p_{\theta}(r|a_t) = e^{r a_t^\top \theta  -  A( a_t^\top \theta) + C(r) }$, with conditional mean as $ A'(a_t^\top \theta)$ and conditional variance as $A''(a_t^\top \theta)$, where $A(\cdot)$ is the log-partition function and $h(\cdot)$ is some known mapping. We also define the link function $g(\cdot):=A'(\cdot)$.  
            \item[ii.] the link function and its derivative are Lipschitz continuous, that is, for any $x,y\in \Omega$, (a) there exists a constant $C_g>0$, such that
             $   |g(x)-g(y)|\leq C_g|x-y|$
             and (b) there exists a constant $L>0$, such that$ |g'(x)-g'(y)|\leq L|x-y|$.
            \item[iii.] the log-partition function is strongly convex with parameter $m$, that is for any $\alpha\in (0,1)$, $x,y\in \Omega$, there exists an $m>0$, such that
                \(\alpha A(x) + (1-\alpha)A(y ) - A(\alpha x + (1-\alpha)y ) \geq \frac{1}{2} \alpha (1-\alpha) m |x-y|^2\).
        \end{itemize}
        
    \end{definition} 

    Typically, if $\max_{x\in \R}|g'(x)|<\infty$, then the condition \textit{ii}(a) above is satisfied for $C_g= \max_{x\in \R}|g'(x)|$. 
    Moreover, when $|A''(\cdot)|\geq m >0$, then $A(\cdot)$ is strongly convex with parameter $m$. Note that conditions (\textit{ii}(a)) and (\textit{iii}) above follow when $|A''(\cdot)|\in [m,C_g]$. This condition is restrictive for exponential family distribution as it requires the variance of many exponential family distributions (such as Poisson or Exponential) to lie in a compact space.  
We show in the Appendix~\ref{app:Exp} that the above exponential family of distributions satisfies all the conditions required for the first and second-order properties of $\alpha$-posterior. Consequently, we have the following corollary of Theorem~\ref{thm:GRB}. 

\begin{corollary}~\label{corr:GRB_Exp}
    For the exponential family reward distributions (Definition~\ref{ass:rewardExp}), under Assumptions~\ref{ass:prior},~\ref{ass:priorSG}, and~\ref{ass:Prior2} on the prior distribution and for some prior dependent positive constant $C_{\phi}$, we show for $\eta= O(d)$ and $\alpha(1-\alpha)\lambda<1$ that $\e_t^2= \frac{4d C_g\log(t)}{D\alpha  C_{\phi}t}$, $\Delta_t(d,\eta) \leq 1$ for sufficiently large $t\geq t_1$, and for $\alpha=d^{-1}$,  $\E[\mReg(T)]= O(d^{3/2}\sqrt{T}\log(T))$. 
\end{corollary}
The expression for $\e_t$ is derived in Lemma~\ref{lem:Exp_prior} and  $\Delta_t(d,\eta)$ is computed in Lemma~\ref{lem:SpokExp}, where we satisfy all the conditions required for the second-order analysis of $\alpha$-posterior for exponential family. 

\subsection{Sub-Gaussian Family}

We begin with the definition of the sub-Gaussian family and state the required conditions.
\begin{definition}[Sub-Gaussian]~\label{ass:reward}
   For any $\Theta\subseteq \R^d$, we assume that 
    \begin{itemize}
        \item[i.] the reward distribution is conditionally sub-Gaussian with bounded sub-Gaussian parameter $\sigma_{\theta}$, that is for any $t\geq 0$, and given  $a_t$, $\E [e^{s r_t}|a_t] \leq e^{s a_t^\top \theta  + \sigma_{\theta}s^2/2}$ for all $s\in \R$,   
        
        \item[ii.] $\sigma_{\theta}$ is bounded by $1$, which implies $\E [e^{s r_t}|a_t] \leq e^{s a_t^\top \theta  + s^2/2}$ for all $s\in \R$, and
        
        \item[iii.] the true mean reward $a^\top \theta_0 $ lies in $[0,1]$ for any $a\in \sA$.
        \end{itemize}
    \end{definition}

We show in the Appendix~\ref{app:SG} that the above sub-Gaussian family of distributions satisfies all the conditions required for the first and second-order properties of $\alpha$-posterior. Consequently, we have the following corollary of Theorem~\ref{thm:GRB} for the sub-Gaussian family. 

\begin{corollary}~\label{corr:GRB_SG}
    For the sub-Gaussian family of reward distributions  (Definition~\ref{ass:reward}) with some additional regularity conditions on the error density, under Assumptions~\ref{ass:prior},~\ref{ass:priorSG}, and~\ref{ass:Prior2} on the prior distribution and for some prior dependent positive constant $C_{\phi}$ and a positive constant $\tilde{C}$, we show for any $\eta=O(d)$ and $\alpha(1-\alpha)\lambda<1$ that $\e_t^2= \frac{4d \tilde{C}\log(t)}{D\alpha  C_{\phi}t}$, $\Delta_t(d,\eta) \leq 1$ for sufficiently large $t\geq t_1$, and for $\alpha=d^{-1}$,  $\E[\mReg(T)]= O(d^{3/2}\sqrt{T}\log(T))$. 
\end{corollary}
We derive the expression for $\e_t$ in Lemma~\ref{lem:SG_prior} and for $\Delta_t(d,\eta)$ in Lemma~\ref{lem:SpokSG}, where we show that sub-Gaussian family satisfies all the conditions required for the second-order analysis of $\alpha$-posterior. A similar regret bound can be computed when the mean of the reward distribution is generalized to $g(a_t^\top \theta)$ by assuming that $g$ satisfies Assumption~\ref{ass:Link}.   

\nomenclature[A]{$\Delta_t(d,\eta)$}{Non-random quantity appearing in the finite-sample BvM bound}
\nomenclature[A]{$\bar a_t$}{Unsaturated action with minimum $V_t^{-1}$-norm}
\nomenclature[A]{$V_t$}{Design matrix (or precision matrix surrogate) at time $t$}
\nomenclature[A]{$\protect \mid\mid x \protect \mid \mid_A$}{Norm induced by a positive semidefinite matrix $A$}
\nomenclature[A]{$\mathcal{C}_t$}{Set of saturated actions at time $t$}

\nomenclature[B]{$\theta_t$}{Parameter sampled from the $\alpha$-posterior at time $t$}
\nomenclature[B]{$\bar \theta_t$}{First-order approximation to the maximum likelihood estimator in the finite-sample BvM result}
\nomenclature[B]{$\phi_d$}{Standard Gaussian random vector in $\mathbb{R}^d$}

\nomenclature[R]{$\eta$}{Positive tuning parameter used in high-probability bounds}
\nomenclature[R]{$C(\alpha,\lambda,\theta_0)$}{Constant appearing in posterior concentration bounds}
\nomenclature[R]{$C^\dagger$}{Constant in the finite-sample BvM bound}
\nomenclature[R]{$C_\phi$}{Prior-dependent constant in the example corollaries}
\nomenclature[R]{$\protect \tilde{C}$}{Constant appearing in the sub-Gaussian corollary}
\nomenclature[R]{$L$}{Lipschitz constant of the derivative of the link function}
\nomenclature[R]{$A(\cdot)$}{Log-partition function of the exponential family}
\nomenclature[R]{$h(\cdot)$}{Base measure term in the exponential family density}
\nomenclature[R]{$\Omega$}{Domain on which smoothness conditions for the link function are imposed}
\nomenclature[R]{$\sigma_\theta$}{Sub-Gaussian parameter of the reward distribution under parameter $\theta$}

%% file: Conclusion.tex
\begingroup
\section{Conclusion}\label{sec:Conclusion}

We studied $\alpha$-Thompson Sampling ($\alpha$-TS), a posterior-sampling algorithm for stochastic generalized linear bandit problems that uses a fractional or $\alpha$-posterior instead of the standard posterior distribution. Our main contribution is a general frequentist regret analysis framework built on structural regularity conditions on the prior and reward model, without requiring any tractable closed-form representation of the posterior distribution. Importantly, $\alpha$-TS does not introduce a new algorithmic procedure, but rather provides a principled interpretation of the variance inflation commonly used in existing analyses of Thompson Sampling. For the optimal choice $\alpha = d^{-1}$, our framework recovers the best known regret bound of $O(d^{3/2}\sqrt{T}\log T)$ for both exponential-family and sub-Gaussian reward distributions, matching the guarantees of state-of-the-art Thompson Sampling analyses. The proof technique combines ideas from the linear bandit literature with first- and second-order posterior concentration theory (including a finite-sample Bernstein-von Mises theorem), and delivers regret bounds in expectation rather than in high probability. We also  provide an $\alpha$-dependent lower bound construction illustrating how the regret scales with the tempering parameter $\alpha$ and the dimension $d$, showing that the $d^{3/2}$ dependence in the upper bound is unavoidable for this class of posterior-sampling algorithms.

\textbf{Future directions.} Several directions remain open. First, whether the $\sqrt{d}$ gap between the $O(d^{3/2}\sqrt{T})$ upper bound and the $\Omega(d\sqrt{T})$ minimax lower bound can be closed for posterior-sampling algorithms without additional structural assumptions is an important open problem. Second, extending the framework to misspecified models, non-stationary environments, or infinite-dimensional parameter spaces (e.g., Gaussian process bandits) would broaden its applicability. Third, developing computationally efficient methods for sampling from the $\alpha$-posterior in non-conjugate settings is a natural next step.
\endgroup

%% file: Nomenclature.tex
\begingroup
\section{Table of Notation}\label{app:notation}

\printnomenclature

\endgroup

%% file: Supp.tex
\begingroup
\section{Proof of Proposition~\ref{prop:lowerbound}}~\label{app:Proof_of_lowerbound}
\proof{Proof}
For $\theta_0 \in \{\pm \mu\}^d$, the optimal action is
\(
a_0 = \operatorname{sign}(\theta_0)
\)
(here sign function is applied coordinatewise) and since \(\sA=[-1,1]^d\), the action selected by \(\alpha\)-TS satisfies
\(
a_t=\arg\max_{a\in\sA} a^\top \tilde\theta_t=\operatorname{sign}(\tilde\theta_t)
\)
coordinatewise, up to the zero-probability event \(\tilde\theta_{t,i}=0\).
Thus the instantaneous regret satisfies
\(
\langle \theta_0, a_0 - a_t \rangle
=
2\mu \sum_{i=1}^d
\mathbf{1}(a_{t,i} \neq \operatorname{sign}(\theta_{0,i})).
\)
Taking expectations yields
\(
\mathbb E[R_T]
=
2\mu \sum_{t=1}^T \sum_{i=1}^d
\mathbb P(a_{t,i} \neq \operatorname{sign}(\theta_{0,i})).
\)
Under $\alpha$-TS, the sampled parameter satisfies
\(
\tilde\theta_t \sim \mathcal N(\mu_t^\alpha,(B_t^\alpha)^{-1}),
\) where precision
\(
B_t^\alpha=I+\alpha\sum_{s=1}^{t-1}a_sa_s^\top
\)
and mean
\(
\mu_t^\alpha=(B_t^\alpha)^{-1}\alpha\sum_{s=1}^{t-1}a_s r_s.
\)
Therefore
\(
\tilde\theta_{t,i}
\sim
\mathcal N(\mu_{t,i}^\alpha,((B_t^\alpha)^{-1})_{ii}).
\)
Hence,
\(
\mathbb P(a_{t,i} \neq \operatorname{sign}(\theta_{0,i}) \mid \mathcal F_{t-1})
=
\Phi\!\left(
-\frac{\operatorname{sign}(\theta_{0,i}) \mu_{t,i}^\alpha}
{\sqrt{((B_t^\alpha)^{-1})_{ii}}}
\right).
\)
Using monotonicity of $\Phi$, we obtain
\(
\Phi(-\operatorname{sign}(\theta_{0,i})x)
\ge
\Phi(-|x|).
\)
Thus,
\(
\mathbb P(a_{t,i} \neq \operatorname{sign}(\theta_{0,i}))
\ge
\mathbb E\!\left[
\Phi\!\left(
-\frac{|\mu_{t,i}^{(\alpha)}|}
{\sqrt{((B_t^\alpha)^{-1})_{ii}}}
\right)
\right].
\)

Now substituting \(r_s=a_s^\top\theta_0+\eta_s\) and using
\(
\alpha\sum_{s=1}^{t-1}a_sa_s^\top=B_t^\alpha- I
\)
gives
\(
\mu_t^\alpha
=
\theta_0
-
 (B_t^\alpha)^{-1} \theta_0
+
(B_t^\alpha)^{-1}
\alpha
\sum_{s=1}^{t-1}
a_s \eta_s.
\)
Taking the $i$-th coordinate and applying the triangle inequality gives
\[
|\mu_{t,i}^\alpha|
\le
|\theta_{0,i}|
+
\left|e_i^\top (B_t^\alpha)^{-1}\theta_0\right|
+
\left|
e_i^\top (B_t^\alpha)^{-1}
\alpha\sum_{s=1}^{t-1} a_s \eta_s
\right|.
\]
Since $\theta_0\in\{\pm\mu\}^d$, we have $|\theta_{0,i}|=\mu$. Moreover,
\(
(B_t^\alpha)_{ii}
=
1+\alpha\sum_{s=1}^{t-1} a_{s,i}^2
\le
1+\alpha t,
\)
because $|a_{s,i}|\le 1$. Since $B_t^\alpha$ is positive definite,
\(
((B_t^\alpha)^{-1})_{ii}\ge \frac{1}{(B_t^\alpha)_{ii}},
\)
and hence
\(
\frac{|\theta_{0,i}|}{\sqrt{((B_t^\alpha)^{-1})_{ii}}}
\le
\mu\sqrt{1+\alpha t}.
\)
For the second term, Cauchy--Schwarz yields
\(
\left|e_i^\top (B_t^\alpha)^{-1}\theta_0\right|
\le
\|\theta_0\|_2\sqrt{((B_t^\alpha)^{-1})_{ii}}
=
\mu\sqrt{d}\,\sqrt{((B_t^\alpha)^{-1})_{ii}}.
\) For the noise term, conditional on $\mathcal F_{t-1}$, the quantity
\[
e_i^\top (B_t^\alpha)^{-1}\sum_{s=1}^{t-1} a_s \eta_s
\]
is Gaussian with mean zero, since it is a linear combination of the independent Gaussian noises $\{\eta_s\}_{s=1}^{t-1}$. Its conditional variance is
\(
e_i^\top (B_t^\alpha)^{-1}
\Big(\sum_{s=1}^{t-1} a_s a_s^\top\Big)
(B_t^\alpha)^{-1}e_i.
\)
Using
\(
B_t^\alpha = I+\alpha\sum_{s=1}^{t-1} a_s a_s^\top,
\)
we obtain
\(
\sum_{s=1}^{t-1} a_s a_s^\top
=
\alpha^{-1}(B_t^\alpha-I),
\)
and therefore the conditional variance is
\(
\alpha^{-1}
e_i^\top (B_t^\alpha)^{-1}(B_t^\alpha-I)(B_t^\alpha)^{-1}e_i
\le
\alpha^{-1} e_i^\top (B_t^\alpha)^{-1} e_i
=
\frac{((B_t^\alpha)^{-1})_{ii}}{\alpha}.
\)
Hence there exists a conditionally centered Gaussian random variable $Z_{t,i}$ with variance at most $1$ such that
\(
e_i^\top (B_t^\alpha)^{-1}\sum_{s=1}^{t-1} a_s \eta_s
=
\sqrt{\frac{((B_t^\alpha)^{-1})_{ii}}{\alpha}}\; Z_{t,i}.
\)
Multiplying by $\alpha$ yields
\[
\left|
e_i^\top (B_t^\alpha)^{-1}
\alpha\sum_{s=1}^{t-1} a_s \eta_s
\right|
\le
\sqrt{((B_t^\alpha)^{-1})_{ii}}\,
\sqrt{\alpha}\,|Z_{t,i}|.
\]
Combining these bounds yields
\[
\frac{|\mu_{t,i}^{(\alpha)}|}
{\sqrt{((B_t^\alpha)^{-1})_{ii}}}
\le
\mu\sqrt{1+\alpha t}
+
\mu\sqrt{d}
+
\sqrt{\alpha}|Z_{t,i}|.
\] \qed
\endproof
\endgroup

\section{Proof of Theorem~\ref{thm:GRB}}~\label{app:Proof}

\proof{Proof of Lemma~\ref{lem:post}:}
    Let us first define a set for any $D>0$, $j>0$, and $a^{(t)}$ adapted to $\gF_t$ as
 \(   F_{t,\e_t} : = \left\{\theta \in \Theta: \frac{\alpha\lambda}{2} \|\theta -\theta_0\|^2 + D_{\alpha}^{(t)}(\theta,\theta_0) \geq \frac{(D+j)\alpha}{(1-\alpha)2} t\e_t^2  \right\}.\)
Note that the $F_{t,\e_t}$ is adapted to $\gF_t$ because the definition of $D_{\alpha}^{(t)}(\theta,\theta_0)$ requires knowledge of $a^{(t)}$. The $\alpha$-posterior measure of the set $F_{t,\e_t}$ can be defined as
\begin{align}
\Pi_{t,\alpha}(F_{t,\e_t}|\gF_t)&= \frac{ \int_{F_{t,\e_t}}  \prod_{s=1}^{t} [p_{\theta}(r_s|a_s)]^{\alpha}\Pi(d\theta) } { \int  \prod_{s=1}^{t} [p_{\theta}(r_s|a_s)]^{\alpha} \Pi(d\theta)  }
    = \frac{\int_{F_{t,\e_t}} \Pi(d\theta) e^{-\alpha \ell_{t}(\theta,\theta_0)} } {\int \Pi(d\theta) e^{-\alpha \ell_{t}(\theta,\theta_0)} },
    \label{eq:MABDE3}
\end{align}
where $\ell_{t}(\theta,\theta_0)= \log \frac{\prod_{s=1}^{t} p_0(r_s|\gF_{s-1},a_s)}{\prod_{s=1}^{t} p_{\theta}(r_s|a_s) }$.
Now define a set 
\(A_{t,\e_t} = \left\{ X_{t}: \int_{\Theta} \tilde \Pi(d\theta) e^{-\alpha \ell_{t}(\theta,\theta_0)} \leq   e^{-\frac{(D+j)\alpha}{4} t\e_t^2}  \right\},\) where for any $B\subseteq \Theta$, $\tilde \Pi(B) = \Pi(B\cap \Theta)/\Pi(B\left(\theta_0,\e_t \right))$ and $B\left(\theta_0,\e_t \right)$ is as defined in Assumption~\ref{ass:prior}. (Note: Assumption~\ref{ass:prior} is a stronger statement as it needs to be satisfied for any $a^{(t)}$; however, for this lemma, we need this assumption only for $a^{(t)}$ that are adapted to $\gF_t$.)
Observe that
\begin{align}
\E [\Pi_{t,\alpha}(F_{t,\e_t}|\gF_{t}) ] =  P_0^{(t)} (A_{t,\e_t}  ) + \E \left[  \frac{\int_{F_{t,\e_t}} \Pi(d\theta) e^{-\alpha \ell_{t}(\theta,\theta_0)} } {\int \Pi(d\theta) e^{-\alpha \ell_{t}(\theta,\theta_0)} } \1(A_{t,\e_t}^\C)   \right].
\label{eq:MABDE4}
\end{align}

First, let us analyze the second term in~\eqref{eq:MABDE4}. Note that on the set $A_{t,\e_t}^\C$,
\(\int \Pi(d\theta) e^{-\alpha \ell_{t}(\theta,\theta_0)} \geq \Pi(B\left(\theta_0,\e_t \right))  e^{-\frac{(D+j)\alpha}{4} t \e_t^2} .\) Therefore, using  Assumption~\ref{ass:prior}, it follows that
\begin{align}
\nonumber
\E \left[  \frac{\int_{F_{t,\e_t}} \Pi(d\theta) e^{-\alpha \ell_{t}(\theta,\theta_0)} } {\int \Pi(d\theta) e^{-\alpha \ell_{t}(\theta,\theta_0)}  } \1(A_{t}^c)   \right] &\leq e^{\frac{(D+j)\alpha}{4} t \e_t^2}
 \E \left[  \frac{{\int_{F_{t,\e_t}} \Pi(d\theta) e^{-\alpha \ell_{t}(\theta,\theta_0)}}}{{\Pi(B\left(\theta_0,\e_t \right)) }}     \right] 
 \\
 &\leq e^{(\frac{j\alpha}{4} + \frac{D\alpha}{2}) t \e_t^2  }
 \E \left[  {\int_{F_{t,\e_t}} \Pi(d\theta) e^{-\alpha \ell_{t}(\theta,\theta_0)}}    \right].
\label{eq:MABDE5}
\end{align}
Now it follows from the definition of $\alpha$-R\'enyi divergence that
\(\E\left[e^{-\alpha \ell_{t}(\theta,\theta_0)}|a^{(t)}\right]= 
e^{-(1-\alpha) D_{\alpha}^{(t)}(\theta,\theta_0)}.\)
Hence, using Fubini's theorem and the observation above, it follows that
\begin{align}
\nonumber
\E\left[\E\left[\int_{F_{t,\e_t}} \Pi(d\theta) e^{-\alpha \ell_{t}(\theta,\theta_0)} \right]|a^{(t)}\right]
&= 
    \E\left[ \int_{F_{t,\e_t}} \Pi(d\theta) e^{-(1-\alpha)  D_{\alpha}^{(t)}(\theta,\theta_0)}\right]
\\
\nonumber
&\leq \E\left[ \int_{F_{t,\e_t}} \Pi(d\theta) e^{-(1-\alpha)  \left(\frac{(D+j)\alpha}{(1-\alpha)2} t\e_t^2  - \frac{\alpha \lambda}{2} \|\theta-\theta_0\|^2 \right)}\right]
\\
&\leq 
     e^{- \frac{(D+j) \alpha}{2}  t\e_t^2 } \E\left[ \int_{F_{t,\e_t}} \Pi(d\theta) e^{\frac{\alpha (1-\alpha)  \lambda}{2} \|\theta-\theta_0\|^2}\right],
     \label{eq:MABDE6}
\end{align}
where the penultimate inequality is due to the definition of  $F_{t,\e_t}$.
Substituting~\eqref{eq:MABDE6} into~\eqref{eq:MABDE5} yields,
\begin{align}
\E \left[  \frac{\int_{F_{t,\e_t}} \Pi(d\theta) e^{-\alpha \ell_{t}(\theta,\theta_0)}} {\int \Pi(d\theta) e^{-\alpha \ell_{t}(\theta,\theta_0)} } \1(A_{t}^c)   \right] 
&\leq e^{ -\frac{j \alpha }{4} t\e_t^2 } \int \Pi(d\theta) e^{\frac{\alpha (1-\alpha)  \lambda}{2} \|\theta-\theta_0\|^2}.
\label{eq:MABDE7}
\end{align}
Next, we analyze the first term in~\eqref{eq:MABDE4}.  It follows from the Markov inequality that
\begin{align}
\nonumber
P_0^{(t)}& \left( \left[ \int \tilde\Pi(d\theta) e^{-\alpha \ell_{t}(\theta,\theta_0)}\right]^{-1/\alpha} \geq e^{ \frac{D+j}{4} t\e_t^2 } \right) 
\\
\nonumber
& \leq  {e^{ - \frac{D+j}{4} t\e_t^2 } }\E \left( \left[ \int \tilde \Pi(d\theta) e^{-\alpha \ell_{t}(\theta,\theta_0)} \right]^{-1/\alpha}   \right)
\\
\nonumber
&\leq  {e^{ -  \frac{D+j}{4}  t\e_t^2 } } \E\left[ \int \tilde \Pi(d\theta) \E \left( e^{\ell_{t}(\theta,\theta_0)}    |a^{(t)} \right) \right]
\\
\nonumber
& =   {e^{ -  \frac{D+j}{4}  t\e_t^2 } } \E\left[ \int \tilde \Pi(d\theta) e^{ D_2^{(t)} \left( \theta_0, \theta \right)} \right]
\\
&\leq   {e^{ - \frac{(D+j)\alpha }{4}  t\e_t^2 } e^{ \frac{D \alpha }{4} t\e_t^2 } }
=  e^{ - \frac{j\alpha}{4}  t\e_t^2 },
\label{eq:MABDE8}
\end{align}
where the second inequality is due to Jensen's and Fubini's theorems, and the penultimate inequality uses the definition of the set $B\left(\theta_0,\e_t \right)$ and the fact that $\alpha\in(0,1)$.
Combining~\eqref{eq:MABDE7} and~\eqref{eq:MABDE8}, it follows from~\eqref{eq:MABDE4}  that for all $j>0$ and $\alpha\in (0,1)$,
\begin{align}
\nonumber
\E &\left[ E_{\alpha}\left[\1\{ \frac{\lambda \alpha}{2}\|\theta-\theta_0\|^2 + D_{\alpha}^{(t)}(\theta,\theta_0) \geq \frac{(D+j)\alpha}{(1-\alpha)2} t\e_t^2 \} | \gF_{t} \right] \right]  
\\
&= \E [\Pi_{t,\alpha}(F_{t,\e_t}|\gF_{t})  ]
\leq 2 e^{ - \frac{j\alpha}{4}  t\e_t^2 }  \int \Pi(d\theta) e^{\frac{\alpha (1-\alpha)  \lambda}{2} \|\theta-\theta_0\|^2}.
\label{eq:MABDE9}
\end{align}
This proves the first assertion of the lemma.
Now for the second assertion, note that the RHS above is non-increasing in $j$; therefore it follows from the inequality above that for all $s\geq 1$,
\begin{align}
\nonumber
\E &\left[ E_{\alpha}\left[\1\{\frac{\lambda \alpha}{2}\|\theta-\theta_0\|^2 + D_{\alpha}^{(t)}(\theta,\theta_0) \geq \frac{(D+s)\alpha}{(1-\alpha)2} t\e_t^2 \} | \gF_{t} \right] \right]  
\\
&\leq  2 e^{ - \frac{(s-1)\alpha}{4}  t\e_t^2 } \int \Pi(d\theta) e^{\frac{\alpha (1-\alpha)  \lambda}{2} \|\theta-\theta_0\|^2}.
\label{eq:MABDE10}
\end{align}
Since $\frac{\lambda \alpha}{2}\|\theta-\theta_0\|^2 + D_{\alpha}^{(t)}(\theta,\theta_0)>0$, it is straightforward to see that
\begin{align}
\nonumber
E\left[\frac{\lambda \alpha}{2}\|\theta-\theta_0\|^2 + D_{\alpha}^{(t)}(\theta,\theta_0) | \gF_{t} \right] &= \int_{0}^{\infty} E_{\alpha}\left[\1\{\frac{\lambda \alpha}{2}\|\theta-\theta_0\|^2 + D_{\alpha}^{(t)}(\theta,\theta_0) \geq u \} | \gF_{t} \right] du
\\
\nonumber
\leq \frac{(D+2)\alpha}{(1-\alpha)2} t\e_t^2 + \int_{\frac{(D+2)\alpha}{(1-\alpha)2} t\e_t^2}^{\infty} &E_{\alpha}\left[\1\{\frac{\lambda \alpha}{2}\|\theta-\theta_0\|^2 + D_{\alpha}^{(t)}(\theta,\theta_0) \geq u \} | \gF_{t} \right]  du.
\end{align}
Now using Fubini's theorem, we have
\begin{align}
\nonumber
\E&\left[ E_{\alpha}\left[\frac{\lambda \alpha}{2}\|\theta-\theta_0\|^2 + D_{\alpha}^{(t)}(\theta,\theta_0) | \gF_{t} \right]\right]
\\
\nonumber
\leq&
\frac{(D+2)\alpha t\e_t^2}{(1-\alpha)2}  + \int_{\frac{(D+2)\alpha}{(1-\alpha)2} t\e_t^2}^{\infty} \E \left[ E_{\alpha}\left[\1\{ \frac{\lambda \alpha}{2}\|\theta-\theta_0\|^2 + D_{\alpha}^{(t)}(\theta,\theta_0) \geq u \} | \gF_{t} \right]\right]  du
\\
\nonumber
=&\frac{(D+2)\alpha t\e_t^2}{(1-\alpha)2}  + \int_{2}^{\infty} \E  E_{\alpha}\left[\1\{\frac{\lambda \alpha \|\theta-\theta_0\|^2}{2} + D_{\alpha}^{(t)}(\theta,\theta_0) \geq \frac{(D+s)\alpha t\e_t^2}{(1-\alpha)2}  \} | \gF_{t} \right] \frac{\alpha t\e_t^2 ds}{(1-\alpha)2}  
\\
\nonumber
\leq& \frac{(D+2)\alpha}{(1-\alpha)2} t\e_t^2 + \frac{\alpha t\e_t^2}{(1-\alpha)}   e^{ - \frac{\alpha}{4}  t \e_t^2 } \int \Pi(d\theta) e^{\frac{\alpha (1-\alpha)  \lambda}{2} \|\theta-\theta_0\|^2} \int_{2}^{\infty}  e^{ - \frac{(s-2)\alpha}{4}  t \e_t^2 } ds 
\\
\leq& \frac{(D+2)\alpha}{(1-\alpha)2} t\e_t^2 + \frac{4e^{ - \frac{\alpha}{4}  t \e_t^2 } }{(1-\alpha)} \int \Pi(d\theta) e^{\frac{\alpha (1-\alpha)  \lambda}{2} \|\theta-\theta_0\|^2}, 
\end{align}
where the penultimate inequality is due to~\eqref{eq:MABDE10} and the last inequality holds for all $\alpha t\e_t^2/4 \geq 0.6$. Now the result follows by using the observation that $\int \Pi(d\theta) e^{\frac{\alpha (1-\alpha)  \lambda}{2} \|\theta-\theta_0\|^2} \leq C_{\pi} f e^{\frac{\alpha (1-\alpha)  \lambda}{2} \|\theta_0\|^2} := C(\alpha,\lambda,\theta_0)$ for some increasing mapping $f$ of $\alpha (1-\alpha)\lambda\|\theta_0\|$.\halmos{} 
\endproof

The next result computes a lower bound on the probability of sampling from the posterior along the best action sequence. 
\begin{lemma}~\label{lem:SpokGRB}
    Under the conditions of Lemma~\ref{lem:FBVM} for $\bar \theta_t=  \theta_0 +  \sD_0^{-2}\nabla\sL( \theta_0)$, we have 
    \begin{align}
        \nonumber
      \Pi_{\alpha,t}&(a_0^\top  \theta_t > a_0^\top\theta_0)|\gF_{t-1}) = \Pi_{\alpha,t}(a_0^\top \sD_0^{-1} \sD_0\left( \theta_t-\bar \theta_t\right) > -a_0^\top \sD_0^{-2}\nabla\sL( \theta_0)|\gF_{t-1}) 
        \\
        &\geq \mathbb{C}\exp\left(-2\alpha\Delta(d,\eta)  - 8e^{-\eta}\right)  \left\{ \frac{{C^\dagger}\sqrt{\alpha (d+\eta)}}{1+{C^\dagger}^2{\alpha (d+\eta)}}e^{-{C^\dagger}^2{\alpha (d+\eta)}/2}  \right\}-e^{-\eta}
    \end{align}
    with $P_{0,a}^{(t)}$ probability of at least $1-2e^{-\eta}$ for any $\eta>0$.
\end{lemma}
\proof{Proof of Lemma~\ref{lem:SpokGRB}}
   Using Lemma~\ref{lem:FBVM}, we have
\begin{align}
    \nonumber
    \Pi_{\alpha,t}&(a_0^\top \sD_0^{-1} \sD_0\left( \theta_t-\bar \theta_t\right) > -a_0^\top \sD_0^{-2}\nabla\sL( \theta_0)|\gF_{t-1})
    \\ 
    \nonumber
    &\geq \Pi_{\alpha,t}(a_0^\top \sD_0^{-1} \sD_0\left( \theta_t-\bar \theta_t\right) > \|a_0^\top \sD_0^{-1}\|\|\sD_0^{-1} \nabla\sL( \theta_0)\||\gF_{t-1})
    \\
    \nonumber
    &\geq \Pi_{\alpha,t}(a_0^\top \sD_0^{-1} \sD_0\left( \theta_t-\bar \theta_t\right) > C \|a_0^\top \sD_0^{-1}\|\sqrt{d+\eta}|\gF_{t-1})
   \\
   \nonumber
    & \geq \mathbb{C}\exp\left(-2\alpha\Delta_t(d,\eta)  - 8e^{-\eta}\right) \left\{P(\phi_d/\sqrt{\alpha} \in A)  \right\} -e^{-\eta}
    \\
    &=  \mathbb{C}\exp\left(-2\alpha\Delta_t(d,\eta)  - 8e^{-\eta}\right) \left\{P(\phi  \geq {C^\dagger} \sqrt{\alpha(d+\eta)})  \right\} -e^{-\eta}
    \\
    &\geq \mathbb{C}\exp\left(-2\alpha\Delta_t(d,\eta)  - 8e^{-\eta}\right)  \left\{ \frac{{C^\dagger}\sqrt{\alpha (d+\eta)}}{1+{C^\dagger}^2{\alpha (d+\eta)}}e^{-{C^\dagger}^2{\alpha (d+\eta)}/2}  \right\}-e^{-\eta},
\end{align}
with $P_{0,a}^{(t)}$ probability of at least $1-2e^{-\eta}$ for any $\eta>0$, where $A:= \{x\in \R^d:  a_0^\top \sD_0^{-1} x \geq {C^\dagger} \|a_0^\top \sD_0^{-1}\|\sqrt{d+\eta} \}$. Since, $\Delta_t(d,\eta)$ is a decreasing function of $t$, the result follows for $\Delta(d,\eta):=\Delta_1(d,\eta)\geq \Delta_t(d,\eta)$.
\halmos{} \endproof

\proof{Proof of Theorem~\ref{thm:GRB}:}
    Let $\bar a_t:= \arg\min_{a \notin \mathcal{C}_t} \|a\|_{V_t^{-1}} $ and observe that 
    \begin{align}
    \nonumber
        &\left( g(a_0^\top\theta_0) -   g(a_t^\top\theta_0) \right) 
        \\
        &= \left( g(a_0^\top\theta_0) -    g(\bar a_t^\top\theta_0) \right) + \left( g(\bar a_t^\top\theta_0) -   g(a_t^\top \theta_t) \right) + \left( g(a_t^\top\theta_t) -   g(a_t^\top \theta_0 )\right) .
        \label{eq:GB0}
    \end{align}
    Using the definition of the unsaturated actions~\ref{def:US} observe that the first term in~\eqref{eq:GB4} is bounded as
    \begin{align}
        \nonumber
        \sum_{t=1}^{T} \E\left[\left( g(a_0^\top\theta_0) -   g(\bar a_t^\top \theta_0) \right)\right] \leq {C_g} \sum_{t=1}^{T} \sqrt{t} \e_t &\E[\|\bar a_t\|_{V_t^{-1}}] \leq \sum_{t=1}^{T} \sqrt{t} \e_t \E[\|\bar a_t\|_{V_t^{-1}}^2]^{1/2} 
        \\
         \leq& {C_g} \left(\sum_{t=1}^{T} t \e_t^2 \right)^{1/2} \left(\sum_{t=1}^{T} \E[\|\bar a_t\|_{V_t^{-1}}^2] \right)^{1/2},
         \label{eq:GB1}
    \end{align} 
    where the last inequality uses Cauchy-Schwarz (CS) inequality.
    Consider the third term in~\eqref{eq:GB4} and note that 
    \begin{align}
       \nonumber
       \sum_{t=1}^T \E \left[   (g(a_t^\top\theta_t) -   g(a_t^\top \theta_0 ) )  \right] 
      & \leq C_g \sum_{t=1}^T \E \left[  \left| (a_t^\top\theta_t -   a_t^\top \theta_0 )\right|  \right] 
       \\
       \nonumber
       & \leq C_g\sum_{t=1}^T \E \left[ \|a_t\|_{V_t^{-1}} \|\theta_t -\theta_0 \|_{V_t} \right] 
       \\
       \nonumber
       &\leq  C_g\E \left[\left( \sum_{t=1}^T \|a_t\|_{V_t^{-1}}^2 \right)^{1/2} \left(  \sum_{t=1}^T \|\theta_t -\theta_0 \|_{V_t}^2 \right)^{1/2}\right] 
       \\
       \nonumber
       &\leq  C_g\left( \E \left[ \sum_{t=1}^T \|a_t\|_{V_t^{-1}}^2\right]  \right)^{1/2} \left(   \sum_{t=1}^T  \E \left[\|\theta_t -\theta_0 \|_{V_t}^2 \right] \right)^{1/2}
       \\
       &\leq C_g \sqrt{ 2d \log \left(1+\lambda^{-1} {T}\right)} \left(   \sum_{t=1}^T  \E \left[\|\theta_t -\theta_0 \|_{V_t}^2 \right] \right)^{1/2},
       \label{eq:GB2}
    \end{align}
where the first inequality is due to Assumption~\ref{ass:Link}, the second inequality uses CS with 
    $V_t = \lambda I + \sum_{s=1}^{t-1} a_s a_s^\top$ for any $\lambda>0$, and $\|x\|_{A} = \sqrt{x^TAx}$, the third and fourth inequality are also due to CS. The bound on $\left( \E \left[ \sum_{t=1}^T \|a_t\|_{V_t^{-1}}^2\right]  \right)^{1/2}$ uses elliptical potential lemma from~\cite[Lemma 3.1]{Auer2002} and~\cite[Proposition 1]{carpentier_elliptical_2020}.

    Combined with the fact that $g(a_t^\top \theta_t) > g(\bar a_t^\top \theta_t)$, the second term in~\eqref{eq:GB4} can be bounded in a similar way to obtain 
    \begin{align}
        \sum_{t=1}^{T} \E\left[ \left( g(\bar a_t^\top\theta_0) -   g(a_t^\top \theta_t) \right) \right] \leq C_g \left( \E \left[ \sum_{t=1}^T \|\bar a_t\|_{V_t^{-1}}^2\right]  \right)^{1/2} \left(   \sum_{t=1}^T  \E \left[\|\theta_t -\theta_0 \|_{V_t}^2 \right] \right)^{1/2}
        \label{eq:GB3}
    \end{align}
Now observe that, using the second result in Lemma~\ref{lem:post}, we have 
\begin{align}
    \nonumber
    \sum_{t=1}^T  \E \left[\|\theta_t -\theta_0 \|_{V_t}^2 \right] &= \sum_{t=1}^T  \E \left[(\theta_t -\theta_0)^{\top}  {V_t}(\theta_t -\theta_0) \right]
    \\
    \nonumber
    &= \sum_{t=1}^T  \E \left[\lambda \|\theta_t -\theta_0\|^2 + \sum_{s=1}^{t-1} \|(\theta_t -\theta_0)^{\top} a_s\|^2  \right]
    \\
    \nonumber
    &\leq \sum_{t=1}^T  \E \left[\lambda \|\theta_t -\theta_0\|^2 + \frac{2}{\alpha} D_{\alpha}^{t-1}(\theta_t,\theta_0)  \right]
    \\
    \nonumber
    &\leq \frac{2}{\alpha} \sum_{t=1}^T  \left[\frac{(D+2)\alpha}{(1-\alpha)2} t\e_t^2 + \frac{4e^{ - \frac{\alpha}{4}  t \e_t^2 } }{(1-\alpha)} \int \Pi(d\theta) e^{\frac{\alpha (1-\alpha)  \lambda}{2} \|\theta-\theta_0\|^2} \right]
    \\
    &\leq  \frac{(D+2)C(\alpha,\lambda,\theta_0)}{(1-\alpha)}  \sum_{t=1}^T  t\e_t^2
    \label{eq:GB4},
\end{align}
where the first inequality uses Assumption~\ref{ass:Renyi}. 
(Note: We omit factor $C$ while using Assumption~\ref{ass:Renyi} just for ease of exposition. The arguments can be easily updated to incorporate this factor.) The only term that remains to be analyzed is $\E \left[ \|\bar a_t\|_{V_t^{-1}}^2\right]$. Observe that 
\begin{align}
    \nonumber
    E_{\alpha} \left[ \|a_t\|_{V_t^{-1}}^2|\gF_{t-1}\right] &= E_{\alpha} \left[ \|a_t\|_{V_t^{-1}}^2|\gF_{t-1}, a_t \notin \mathcal{C}_t\right] \Pi_{\alpha,t}(a_t \notin \mathcal{C}_t| \gF_{t-1} )
    \\
    &\geq \|\bar a_t\|_{V_t^{-1}}^2 \Pi_{\alpha,t}(a_t \notin \mathcal{C}_t| \gF_{t-1} ),
    \label{eq:GB5}
\end{align}
where the second inequality uses the definition of $\bar a_t$ and the fact that it is completely determined by $\gF_{t-1}$.
Note that if we can establish a lower bound on  $\Pi_{\alpha,t}(a_t \notin \mathcal{C}_t| \gF_{t-1} )$ then we can upper bound $E_{\alpha} \left[ \|\bar a_t\|_{V_t^{-1}}^2\right]$ by $E_{\alpha} \left[ \|a_t\|_{V_t^{-1}}^2\right]$ upto some time dependent term. 

Since $a_0$ is always unsaturated, therefore if $g(a_0^\top \theta_t) > g(a_j^\top \theta_t)$ for all saturated actions, i.e.$ \forall j\in \mathcal{C}_t$, then one of the unsaturated actions must be played. Consequently,
\begin{align}
    \nonumber
    &\Pi_{\alpha,t}(a_t \notin \mathcal{C}_t| \gF_{t-1} ) \geq P( g(a_0^\top \theta_t) > g(a_j^\top \theta_t), \forall j\in \mathcal{C}_t | \gF_{t-1} ) 
    \\
    \nonumber
    &\geq  \Pi_{\alpha,t} \{ \forall j\in \mathcal{C}_t: g(a_0^\top \theta_t) > g(a_j^\top \theta_t) \} , \{\forall j: g(a_j^\top\theta_0) \geq g(a_j^\top\theta_t) - C_g \sqrt{t} \e_t \|a_j\|_{V_t^{-1}}\}| \gF_{t-1} ) 
    \\
    \nonumber
    &\geq  \Pi_{\alpha,t} \{  g(a_0^\top \theta_t) > g({a_0^\top} \theta_0)   \}  \{\forall j: g(a_j^\top\theta_0) \geq g(a_j^\top\theta_t) - C_g \sqrt{t} \e_t \|a_j\|_{V_t^{-1}}\}| \gF_{t-1} ) 
    \\
    \nonumber
    &\geq  \Pi_{\alpha,t} a_0^\top \theta_t > a_0^\top \theta_0  , \|\theta_t -\theta_0 \|_{V_t}\leq \sqrt{t} \e_t | \gF_{t-1} ) 
    \\
    \nonumber
    & \geq \Pi_{\alpha,t}  (a_0^\top \theta_t > a_0^\top \theta_0  | \gF_{t-1} )  -  \Pi_{\alpha,t}  \|\theta_t -\theta_0 \|_{V_t}\geq \sqrt{t} \e_t | \gF_{t-1} )
    \\
    \nonumber
    &\geq 
    p - \Pi_{\alpha,t} \left( \lambda  \frac{\alpha}{2} \|\theta_t -\theta_0\|^2 +  D_{\alpha}^{t-1}(\theta_t,\theta_0) \geq \frac{\alpha}{2}\sqrt{t} \e_t | \gF_{t-1} \right)
    \\
    &\geq p - 2 e^{ - \frac{\alpha}{4}  t\e_t^2 }  \int \Pi(d\theta) e^{\frac{\alpha (1-\alpha)  \lambda}{2} \|\theta-\theta_0\|^2} \geq p - C(\alpha,\lambda,\theta_0)2 e^{ - \frac{\alpha}{4}  t\e_t^2 } ~\label{eq:eqLB2}
    \end{align}
with $P_0-$ probability of at least $1-2e^{-\eta}$ for $p= \mathbb{C}\exp\left(-2\alpha\Delta(d,\eta)  - 8e^{-\eta}\right)  \left\{ \frac{{C^\dagger}\sqrt{\alpha (d+\eta)}}{1+{C^\dagger}^2{\alpha (d+\eta)}}e^{-{C^\dagger}^2{\alpha (d+\eta)}/2}  \right\}-e^{-\eta}$. {The third inequality follows the definition of saturated arms to imply $ g(a_0^\top \theta_0) \geq g(a_j^\top\theta_0) + C_g \sqrt{t} \e_t \|a_j\|_{V_t^{-1}} \geq g(a_j^\top\theta_t) $.} 
The 
penultimate inequality uses the definition of the saturated actions and the result in~Lemmas~\ref{lem:SpokGRB} and~\ref{lem:post}. So, combining~\eqref{eq:GB5} and~\eqref{eq:eqLB2} and denoting $\bar p = p - C(\alpha,\lambda,\theta_0)2 e^{ - \frac{\alpha}{4}  t\e_t^2 } $ for brevity, we have 
\begin{align}
    E_{\alpha} \left[ \|a_t\|_{V_t^{-1}}^2|\gF_{t-1}\right] &= E_{\alpha} \left[ \|a_t\|_{V_t^{-1}}^2|\gF_{t-1}, a_t \notin \mathcal{C}_t\right] \Pi_{\alpha,t}(a_t \notin \mathcal{C}_t| \gF_{t-1} )
    \geq \|\bar a_t\|_{V_t^{-1}}^2  \bar p
\end{align}
with $P_0-$ probability of at least $1-2e^{-\eta}$. Since, $t\e_t^2$ increases with $t$, for sufficiently large $t$ (say $t_1$), $\bar p$ will be positive. Denoting the above high probability event as $\mathbf{E}$ observe that
\[\E\left[E_{\alpha} \left[ \|a_t\|_{V_t^{-1}}^2|\gF_{t-1}\right]  \right] \geq \E\left[E_{\alpha} \left[ \|a_t\|_{V_t^{-1}}^2|\gF_{t-1}\right] |\mathbf{E} \right] P_0(\mathbf{E}) \geq \E\left[\|\bar a_t\|_{V_t^{-1}}^2\right] \bar p P_0(\mathbf{E}). \]
Therefore, we have 
\begin{align}
\E \left[\sum_{t=1}^{T} \|\bar a_t\|_{V_t^{-1}}^2\right] \leq \frac{1}{\bar p_0 (1-2e^{-\eta}) } \E \left[ \sum_{t=1}^{T} \|a_t\|_{V_t^{-1}}^2\right] \leq \frac{2d \log \left(1+\lambda^{-1} {T}\right)}{\bar p_0 (1-2e^{-\eta}) },
\label{eq:LB}
\end{align}
where 
$\bar p_0 = \mathbb{C}\exp\Big(-2\alpha\Delta(d,\eta) $ 
$ - 8e^{-\eta}\Big)  \left\{ \frac{{C^\dagger}\sqrt{\alpha (d+\eta)}}{1+{C^\dagger}^2{\alpha (d+\eta)}}e^{-{C^\dagger}^2{\alpha (d+\eta)}/2}  \right\}-e^{-\eta} -C(\alpha,\lambda,\theta_0)2 e^{ - \frac{\alpha}{4}  t_0\e_{t_0}^2 }$ for a sufficiently small $t_0 (\geq t_1)$ and the second inequality uses \textit{elliptical potential lemma}. We can choose such $t_0$ because $\bar p$ increases as $t$ increases (since $t \e_t^2$ increases as $t$ increases). 
Now, combining equations~\ref{eq:GB0},~\ref{eq:GB1},~\ref{eq:GB2},~\ref{eq:GB3},~\ref{eq:GB4},~and~\ref{eq:GB5}, we have
\begin{align}
\nonumber
    \E[\mReg(T)] &\leq {C_g} \left(\sum_{t=1}^{T} t \e_t^2 \right)^{1/2} \left(\frac{2d \log \left(1+\lambda^{-1} {T}\right)}{\bar p_0 (1-2e^{-\eta}) } \right)^{1/2} 
    \\
    \nonumber
    &+ C_g \sqrt{ 2d \log \left(1+\lambda^{-1} {T}\right)} \left(   \frac{(D+2)C(\alpha,\lambda,\theta_0)}{(1-\alpha)}  \sum_{t=1}^T  t\e_t^2 \right)^{1/2}
    \\
    &+ C_g\left( \frac{2d \log \left(1+\lambda^{-1} {T}\right)}{\bar p_0 (1-2e^{-\eta}) } \right)^{1/2} \left(   \frac{(D+2)C(\alpha,\lambda,\theta_0)}{(1-\alpha)}  \sum_{t=1}^T  t\e_t^2 \right)^{1/2}.\halmos{} 
\end{align}
\endproof

\input{Exp.tex}

\input{SG.tex}

%% file: Exp.tex
\section{Verifying assumptions for the Exponential family}~\label{app:Exp}

We first provide the expression for the $\alpha$-R\'enyi divergence for exponential family distributions. Observe that for any $\alpha\in(0,1)$,
\(D_{\alpha}^t(\theta,\vartheta) = \frac{1}{1-\alpha}\left[\alpha A(a_t^\top\theta) + (1-\alpha)A(a_t^\top\vartheta ) - A(\alpha a_t^\top\theta + (1-\alpha)a_t^\top\vartheta ) \right].\)
Using this definition, we show that the exponential family satisfies Assumption~\ref{ass:Renyi} {for $C=\sqrt{\frac{1}{m}}$}.

\begin{lemma}~\label{lem:Exp}
    Fix $\alpha\in(0,1)$. For distribution lying in exponential family distribution satisfying conditions in Definition~\ref{ass:rewardExp}, we have,
     \(   |a_t^\top\theta -a_t^\top\vartheta| \leq \sqrt{\frac{1}{m}} \sqrt{\frac{2}{\alpha } D_{\alpha}^{t}(\theta,\vartheta)}. \)
\end{lemma}
\proof{Proof:}
    The proof of this lemma is a direct consequence of the condition~(iii) in Definition~\ref{ass:rewardExp}. \halmos{} 
\endproof

Our next result shows that the exponential family satisfies Assumption~\ref{ass:prior}.
\begin{lemma}~\label{lem:Exp_prior}
    Under Assumption~\ref{ass:action}, the exponential family of distribution as defined in~Definition~\ref{ass:rewardExp}, satisfies Assumption~\ref{ass:prior} for $\e_t^2= \frac{4d C_g\log(t)}{D\alpha  C_{\phi}t} $, where $C_g$ is the Lipschitz constant of the link function and $C_{\phi} = \min_{i\in[d], x\in B(\theta_{0}^i,1)}\phi_i (x)$,  $B(\theta_{0}^i,u)\subset \R$ is a ball of radius $u$ centered at $\theta_{0}^i\in \Theta$ and $\phi_i(\cdot)$ is the density of $\Pi^i_t$.
\end{lemma}
\proof{Proof of Lemma~\ref{lem:Exp_prior}:}
        Observe that 
    \begin{align}
        \nonumber
        D^{(t)}_{2}(\theta, \theta_0) &=  \log \int p_{\theta}^{(t)}(r^{(t)}|a^{(t)})^{2} p_{0}^{(t)}(r^{(t)}|a^{(t)})^{-1} d\mu^{(t)}
        \\
        \nonumber
        &= \sum_{s=1}^{t}  \log \int p_{\theta}(r_s| \gF_{t-1},a_s)^{2} p_{0}(r_s| \gF_{s-1},a_s)^{-1} d\mu
        \\
        \nonumber
        &=  \sum_{s=1}^{t} \left[ A( a_s^\top(2\theta - \theta_0)) - 2 A(a_s^\top\theta ) + A( a_s^\top\theta_0)   \right]
        \\
        \nonumber
        &\leq C_g \sum_{s=1}^{T} \left| (a_s^\top(\theta-\theta_0) \right|^2
        \\
        &\leq t  C_g  \|(\theta-\theta_0)\|^2.
        \label{eq:GLM1}
    \end{align}
    where the third equality uses the definition of $2$-R\'enyi divergence for an exponential family of distributions, the first inequality follows due to condition~(ii) in Definition~\ref{ass:rewardExp} (since $|A''(x)|\leq C_g$, the result follows by second-order mean value theorem), the last inequality follows from Cauchy-Schwartz inequality and the Assumption~\ref{ass:action} that $\sum_{t=1}^{T} \|a_t\|\leq T$.
    Using~\eqref{eq:GLM1} observe that 
    \begin{align}
    \nonumber
    \Pi \left(D^{(t)}_{2}(\theta, \theta_0) \leq \frac{D\alpha}{4}t\e_t^2  \right) &\geq \Pi \left( t C_g  \|\theta-\theta_0)\|^2 \leq \frac{D\alpha}{4}t\e_t^2  \right)
    \\
    \nonumber
    &= \Pi \left( (\theta-\theta_0)^\top (\theta-\theta_0) \leq \frac{D\alpha }{4 C_g }\e_t^2  \right)
    \\
    &\geq \prod_{i=1}^{d} \Pi^i \left( (\theta^i-\theta_0^i)^2 \leq \frac{D \alpha}{4d C_g }\e_t^2  \right).
    \label{eq:exLB0}
    \end{align}
    
Fix $\e_t^2 = \frac{\log(t)}{C^2 t C_{\phi}}$ for $C^2=\frac{D \alpha}{4d C_g }$ and observe that $C\e_t = \sqrt{\frac{\log(t)}{ t C_{\phi} }}\leq 1$ for all $t\geq 1$. Now it follows that $\min_{i\in[d], x\in B(\theta_{0}^i,C\e_T)}\phi (x) \geq \min_{i\in[d], x\in B(\theta_{0}^i,1)}\phi_i (x)  = C_{\phi}$ for all $t\geq t_0$ and any $i\in[d]$, where $B(\theta_{0}^i,u)\subset \R$ is a ball of radius $u$ centered at $\theta_{0}^i\in \Theta$ and $\phi_i(\cdot)$ is the density of $\Pi^i_t$. Note that $C_{\phi}>0$ since the prior density is strictly positive everywhere. Consequently, for all $t\geq t_0$, it follows from our choice of $\e_t$ that
    \begin{align}
            \Pi^i\left( |\theta^i-\theta_{0}^i|^2 \leq \frac{D \alpha}{4d C_g } \e_t^2 \right) \geq 2 \min_{x\in B(\theta_{0}^i,C\e_t)}\phi (x) C\e_t
            & \geq  C_{\phi} \sqrt{\frac{4\log(t)}{ t C_{\phi}^2}} \geq \sqrt{\frac{\log (t)}{t}}.
            \label{eq:exLB1}
        \end{align}
        Now the assertion of the lemma follows using the fact that $\sqrt{\frac{\log n}{n}}\geq \frac{1}{n} =  e^{- {C^2  C_{\phi}} n \e_{n}^2 } =  e^{- \frac{D \alpha C_{\phi}}{4d C_g }  n \e_{n}^2 } $ for any $n\geq 2$. In particular, we have from~\eqref{eq:exLB0},~\eqref{eq:exLB1} and the arguments above that
    \begin{align}
        \Pi\left(D^{(t)}_{2}(\theta, \theta_0)  \leq \frac{D\alpha}{4} t\e_t^2 \right) \geq  \prod_{i=1}^{d} \Pi^i\left( (\theta^i-\theta_0^i)^2 \leq \frac{D \alpha}{4d C_g }\e_T^2  \right)\geq  e^{- \frac{D\alpha  C_{\phi} }{4 C_g} t \e_{t}^2 },
    \end{align}
    and the result follows for any prior for which $C_{\phi}\leq C_g $ (otherwise, this term will appear in the form of a constant in the main result).\halmos{} 
    \endproof

To show that the exponential family satisfies Assumption~\ref{ass:Spokoiny}, the conditions required for the finite BvM to hold, we first write down various expressions used in their definition for exponential family models.
First recall the definition of the conditional log-likelihood of generating $X_t|a^{(t)}$, that is $\sL(\theta) = \log \prod_{s=1}^{t} \left[p_{\theta}(r_s|a_s)\right] $.
Using this definition, the gradient of the log-likelihood can be derived as
\(   \nabla \sL(\theta) = \sum_{s=1}^{t}\nabla \log p_{\theta}(r_s|a_s) = \sum_{s=1}^{t} \left({r_s a_s^\top   -  A'( a_s^\top \theta)a_t^\top }\right).\)
Similarly, $\nabla^2 \sL(\theta) = -\sum_{s=1}^{t} A''( a_t^\top \theta) a_ta_t^\top$.
Recall that $\E_a[]=\E[|a^{(t)}]$.
Now observe that $$\E_a[\nabla \sL(\theta)] = \sum_{s=1}^{t}\E_a[\nabla \sL(\theta)] = \sum_{s=1}^{t}\left(A'( a_s^\top \theta_0)a_s^\top-A'( a_s^\top \theta)a_s^\top\right).$$
and $\sD_0^{2}$ and $\sD_0^{2}(\theta)
     = 
    - \nabla^{2} \E_a[ \sL(\theta)]  = \sum_{s=1}^{t} A''( a_s^\top \theta) a_s a_s^\top$.
The stochastic part of the conditional log-likelihood is denoted as $\zeta (\theta):= \sL(\theta) - \E_a[\sL(\theta)]$ and $\nabla \zeta (\theta) = \nabla \sL(\theta) - \E_a[\nabla \sL(\theta)]= \sum_{s=1}^{t} \left({r_s a_s^\top   -  A'( a_s^\top \theta)a_s^\top }\right) - \E_a[ \sum_{s=1}^{t} \left({r_s a_s^\top   -  A'( a_s^\top \theta)a_s^\top }\right)] = \sum_{s=1}^{t} \left(r_s a_s^\top - \E_a[r_s a_s^\top]\right) = \sum_{s=1}^{t} \left(r_s a_s^\top - A'( a_s^\top \theta_0)a_s^\top\right) $ and therefore $\E_a[\nabla \zeta(\theta_0)]=0$. Note that $\nabla \zeta(\theta)$ is independent of $\theta$, therefore $\nabla^2 \zeta(\theta)=0$. 
First, we present a technical lemma that is satisfied by exponential families.
\begin{lemma}~\label{lem:Aux1}
        There exists some constant $v_0>0$ and $\BFg_1>0$, for every $s$ a constant  $\sigma_s^2=A''(a_s^\top \theta_0)$ (given $\gF_{s-1}$ and $a_s$) such that $\E[(r_s - A'(a_s^\top\theta_0))^2/\sigma_i^2|\gF_{s-1},a_s]\leq 1$ and 
        \(    \log \E [ \exp(\lambda (r_s - A'(a_s^\top\theta_0))/\sigma_s) | \gF_{s-1},a_s ] \leq v_0^2 \lambda^2/2 , ~~|\lambda| < \BFg_1.\)
    \end{lemma}
    \proof{Proof:}
        The proof is a direct consequence of Lemma 2.14~\cite{Spokoiny2012}(in their supplement).
    \halmos{} \endproof

Next, we have the main result that verifies all the conditions in Assumption~\ref{ass:Spokoiny}.
\begin{lemma}~\label{lem:SpokExp}
    The exponential family of distribution (Definition~\ref{ass:rewardExp}) satisfies Assumption~\ref{ass:Spokoiny}. The condition $(ED_0)$ is satisfied with $\BFg:= \BFg_1 \sT^{1/2}$, where $\sT^{-1/2}:= \max_{s\in[t]} \sup_{\gamma \in \R^d} \frac{A''(a_s^\top \theta_0)^{1/2} |a_s^\top \gamma|}{\|\sD_0 \gamma\|} $ and condition $(ED_1)$ can be satisfied for any $\omega>0$ and $\BFg>0$. Condition $\gL_0$ follows for $\delta(\BFr)= L \sT_2^{-1/2} \BFr$, where $\sT^{-1/2}:= \max_{s\in[t]} \sup_{\gamma \in \R^d} \frac{A''(a_s^\top \theta_0)^{1/2} |a_s^\top \gamma|}{\|\sD_0 \gamma\|} $. 
\end{lemma}
\proof{Proof:}
    We derive all the conditions in seriatim. The proofs are adapted from~\cite{Panov2015}, and~\cite{Spokoiny2012}.

  \begin{enumerate}
    \item Proof for \({(E\!D_{0})} \):
    Observe using the definition of $\nabla \zeta(\theta_0)$ that 
        \begin{align}
        \nonumber
            \E_a \exp\left\{
              {\BFm} \frac{\langle \nabla \zeta(\theta_0),\gamma \rangle}
              {\| \sD_0 \gamma \|}
              \right\} &= \E_a \exp\left\{
              {\BFm} \frac{ \sum_{s=1}^{t} \left({r_s a_s^\top\gamma   -  A'( a_s^\top \theta_0)a_s^\top \gamma}\right)}
              {\| \sD_0 \gamma \|}
              \right\}
              \\
              \nonumber
              & = \E_a \prod_{s=1}^t\exp\left\{
              {\BFm} \frac{  a_s^\top \gamma\left({r_s    -  A'( a_s^\top \theta_0)}\right)}
              {\| \sD_0 \gamma \|}
              \right\}
              \\
               = \E_a \Bigg[\prod_{s=1}^{t-1}\exp\left\{
              {\BFm} \frac{  a_s^\top \gamma\left({r_s    -  A'( a_s^\top \theta_0)}\right)}
              {\| \sD_0 \gamma \|}
              \right\} &\E_a\left[\exp\left\{
              {\BFm} \frac{  a_s^\top \gamma\left({r_s    -  A'( a_s^\top \theta_0)}\right)}
              {\| \sD_0 \gamma \|}
              \right\}\right]\Bigg].
              \label{eq:ED0}
        \end{align}

    Define $\sT^{-1/2}:= \max_{s\in[t]} \sup_{\gamma \in \R^d} \frac{A''(a_s^\top \theta_0)^{1/2} |a_s^\top \gamma|}{\|\sD_0 \gamma\|} $. 
    By this definition, $\sT^{-1/2} \geq \frac{A''(a_s^\top \theta_0)^{1/2} |a_s^\top \gamma|}{\|\sD_0 \gamma\|}$ and therefore ${\BFm}\frac{A''(a_s^\top \theta_0)^{1/2} |a_s^\top \gamma|}{\|\sD_0 \gamma\|} \leq {\BFm} \sT^{-1/2} \leq \BFg_1$, for $\BFg:=\BFg_1 \sT^{1/2} $.
    Therefore, using Lemma~\ref{lem:Aux1} it follows that
    \begin{align}
        \E_a\left[\exp\left\{
              {\BFm} \frac{  a_s^\top \gamma\left({r_s    -  A'( a_s^\top \theta_0)}\right)}
              {\| \sD_0 \gamma \|}
              \right\}\right] \leq \exp\left(v_0^2{\BFm}^2 \frac{A''(a_s^\top \theta_0)^2 |a_s^\top \gamma|^2}{2\|\sD_0 \gamma\|^2}\right).
              \label{eq:ED1}
    \end{align}

    Now substituting~\eqref{eq:ED1} into~\eqref{eq:ED0}, we have for $|{\BFm}|\leq \BFg:= \BFg_1 \sT^{1/2}$,
    \begin{align}
    \nonumber
        \E_a &\exp\left\{
              {\BFm} \frac{\langle \nabla \zeta(\theta_0),\gamma \rangle}
              {\| \sD_0 \gamma \|}
              \right\} 
              \\
              \nonumber
              &\leq \E_a \Bigg[\prod_{s=1}^{t-1}\exp\left\{
              {\BFm} \frac{  a_s^\top \gamma\left({r_s    -  A'( a_s^\top \theta_0)}\right)}
              {\| \sD_0 \gamma \|}
              \right\} \exp\left(v_0^2{\BFm}^2 \frac{A''(a_t^\top \theta_0)^2 |a_t^\top \gamma|^2}{2\|\sD_0 \gamma\|^2}\right)\Bigg]
              \\
              &\leq  \exp\left(v_0^2{\BFm}^2 \frac{\sum_{s=1}^{t} A''(a_s^\top \theta_0)^2 |a_s^\top \gamma|^2}{2\|\sD_0 \gamma\|^2}\right) = \exp\left(v_0^2{\BFm}^2/2\right).
    \end{align}
    The result follows for any $\gamma$, so it must follow for the supremum over $\gamma\in \R^d$.

    \item Proof for \( {(E\!D_{1})} \):
      This assumption follows immediately for any $\omega>0$ because $\nabla^2\zeta(\theta) =0$ by definition. 
  
    \item Proof for \({(\gL_0)} \):
      For $I_d= \sD_0^{-1} \sD_0^2\sD_0^{-1}$, we have by the definition of operator norm
      \begin{align}
      \nonumber
          \|\sD_0^{-1} (\sD_0^2(\theta)-\sD_0^2)\sD_0^{-1}\|&=\sup_{\gamma\in \R^d:\|\gamma\|=1} |\gamma^\top  \sD_0^{-1} (\sD_0^2(\theta)-\sD_0^2)\sD_0^{-1} \gamma|
          \\
          \nonumber
          &= \sup_{\gamma\in \R^d:\|\gamma\|=1} \left|\sum_{s=1}^{t} ( A''( a_s^\top \theta) - A''( a_s^\top \theta_0)) \gamma^\top  \sD_0^{-1}a_s a_s^\top\sD_0^{-1} \gamma\right|
          \\
          \nonumber
          &\leq \sup_{\gamma\in \R^d:\|\gamma\|=1} \sum_{s=1}^{t} \left|A''( a_s^\top \theta) - A''( a_s^\top \theta_0))\right| \gamma^\top  \sD_0^{-1}a_s a_s^\top\sD_0^{-1} \gamma
          \\
          &\leq \sup_{\gamma\in \R^d:\|\gamma\|=1} \sum_{s=1}^{t} L \left|a_s^\top (\theta -  \theta_0)\right| \gamma^\top  \sD_0^{-1}a_s a_s^\top\sD_0^{-1} \gamma
      \end{align}
      Define $\sT_2^{-1/2} : = \max_s \sup_{\gamma \in \R^d} \frac{|a_s^\top \gamma|}{A''(a_s^\top \theta_0)\|\sD_0\gamma\|} \geq \frac{|a_s^\top (\theta-\theta_0)|}{A''(a_s^\top \theta_0)\|\sD_0(\theta-\theta_0)\|}  $ and observe that
      \begin{align}
      \nonumber
          \|\sD_0^{-1} (\sD_0^2(\theta)-\sD_0^2)\sD_0^{-1}\|
          &\leq \sup_{\gamma\in \R^d:\|\gamma\|=1} \sum_{s=1}^{t} L \left|a_s^\top (\theta -  \theta_0)\right| \gamma^\top  \sD_0^{-1}a_s a_s^\top\sD_0^{-1} \gamma
          \\
          \nonumber
          &\leq L \sT_2^{-1/2} \|\sD_0(\theta-\theta_0)\| \sup_{\gamma\in \R^d:\|\gamma\|=1}  \gamma^\top  \sD_0^{-1} \sum_{s=1}^{t} A''(a_s^\top \theta_0) a_s a_s^\top\sD_0^{-1} \gamma 
          \\
          &= L \sT_2^{-1/2} \BFr.
      \end{align}
      Consequently, $\delta(\BFr)= L \sT_2^{-1/2} \BFr . $

      \item Proof for \( {(\gL{\BFr})} \):
      For any \(\BFr>0\) there exists a value \({\BFb}(\BFr) > 0\),
      such that  $\BFr {\BFb}(\BFr) \to \infty$ as $\BFr \to \infty$ and 
      \(  -\E_a \sL(\theta,\theta_0)
         \ge
         \BFr^{2} b(\BFr) \quad \text{for all \( \theta \) with } 
         \BFr = \|\sD_0 (\theta - \theta_0)\|.\)
          Note that 
          \begin{align}
          \nonumber
            &\E_a[\sL(\theta,\theta_0)]= \E_a[\sL(\theta)-\sL(\theta_0)] = \sum_{s=1}^{t}\E_a\left[\log\left(\frac{p_{\theta}(r_s|a_s)}{p_{\theta_0}(r_s|\gF_{s-1},a_s)}\right)\right] = \sum_{s=1}^{t} \E_a\Big[ r_s a_s^\top (\theta-\theta_0) 
              \\
              \nonumber
              &  -  A( a_s^\top \theta)  +  A( a_s^\top \theta_0) \Big] 
              =  \sum_{s=1}^{t} \left[A( a_s^\top \theta_0)-  A( a_s^\top \theta) + A'(a_s^\top \theta_0) a_s^\top (\theta-\theta_0 ) \right]
              \\
              &= - \sum_{s=1}^{t} A''(a_s^\top\theta^*)(\theta-\theta_0)^\top a_s a_s^\top(\theta-\theta_0) 
              = -\|\sD_0(\theta^*)(\theta-\theta_0)\|^2,
          \end{align}
          where the penultimate equality uses second-`order Taylor's theorem for a $\theta^*$ that lies between $\theta$ and $\theta_0$. Now observe that $\frac{- \E_a[\sL(\theta,\theta_0)] }{\BFr^2}= \frac{\|\sD_0(\theta^*)(\theta-\theta_0)\|^2}{\|\sD_0(\theta-\theta_0)\|^2}  $. Note that, by definition $\theta^*$ is a mapping from $\BFr$ and so is $\frac{\|\sD_0(\theta^*)(\theta-\theta_0)\|^2}{\|\sD_0(\theta-\theta_0)\|^2} $. Therefore, we can always find a mapping ${\BFb}(\BFr)$ such that  $\BFr{\BFb}(\BFr) \to \infty$  as $\BFr \to \infty$ and $\frac{- \E_a[\sL(\theta,\theta_0)] }{\BFr^2}$. \halmos{}
  \end{enumerate}
  \endproof

Now recall the definition of $\Delta(\BFr_0,\eta)$ from~\cite[Theorem 9]{Panov2015}, that is $\Delta(\BFr_0,\eta)=\left\{\delta(\BFr_0) + 6 v_0z_{\sH}(\eta)\omega\right\}\BFr_0^2$, where $z_{\sH}(\eta):= 2\sqrt{d} + \sqrt{2\eta} + \BFg^{-1}(\BFg^{-2}\eta+1)4d$. Note that, we denote $\Delta(\BFr_0,\eta)$ as $\Delta_t(d,\eta)$ to explicitly show the dependence of $d$, $t$, and $\eta$. Since $\omega$ (from $(E\!D1)$) can be any positive number, we choose $\omega= \frac{\delta(\BFr_0)}{6 v_0z_{\sH}(\eta)}$. Therefore, $\Delta(\BFr_0,\eta)= 2\delta(\BFr_0) \BFr_0^2 \leq L \frac{\BFr_0^3}{\sT_2^{1/2}}$. Moreover, $\BFr_0^2 = C(\eta+d)$ for some known constant $C$~\cite[Section 5.2]{Spokoiny2012}. 
Note that $\|a_s\|= 1$ because $g$ is monotonic and the optimizer of $a^\top\theta$ on $a\in \R^d:\|a\|\leq 1$ is nothing but $\theta/\|\theta\|$, which lies on a d-dimensional sphere ($\gS^{d-1}$). 
 Observe that for some $\gamma^*\in \R^d \backslash\{0\}$, $\sT^{-1/2}:= \max_{s\in[t]} \sup_{\gamma \in \R^d} \frac{A''(a_s^\top \theta_0)^{1/2} |a_s^\top \gamma|}{\|\sD_0 \gamma\|} = \max_{s\in[t]}  \frac{A''(a_s^\top \theta_0)^{1/2} |a_s^\top \gamma^*|}{\|\sD_0 \gamma^*\|}  \leq \frac{1}{\sqrt{t}}\frac{\max_{s\in[t]} A''(a_s^\top \theta_0)^{1/2} |a_s^\top \gamma^*|}{\min_{s\in[t]} A''(a_s^\top \theta_0)^{1/2} |a_s^\top \gamma^*|}\leq \frac{1}{\sqrt{t}} \sqrt{\frac{C_g}{m}} \frac{\max_{s\in[t]}  |a_s^\top \gamma^*|}{\min_{s\in[t]:|a_s^\top \gamma^*|\neq 0}  |a_s^\top \gamma^*|} \leq \frac{1}{\sqrt{t}} \sqrt{\frac{C_g}{m}} \frac{  \|\gamma^*\|}{\min_{s\in[t]:|a_s^\top \gamma^*|\neq 0}  |a_s^\top \gamma^*|} $. Since, $\gamma^*\neq 0$ and $a_s\in \gS^{d-1}$, we can compute a $t$-independent lower bound on  $\min_{s\in[t]:|a_s^\top \gamma^*|\neq 0}  |a_s^\top \gamma^*| \geq {\min_{a\in \gS^{d-1}:|a^\top \gamma^*|\neq 0}  |a_s^\top \gamma^*|}$. Therefore, $\sT^{-1/2} =O(t^{-1/2})$.
By definition of $\sD_0^2$, observe that $\|\sD_0 \gamma^*\|^2= \gamma^\top \sD_0^2 \gamma $ \\ $=  \sum_{s=1}^{t}A''(a_s^\top \theta_0) |\gamma^\top a_s| |a_s^\top \gamma|\geq A''(a_s^\top \theta_0)  |a_s^\top \gamma|^2$. Choosing $t_1$ large enough and fixing $\eta=O(d)$, we have $\Delta_t(d,\eta)\leq 1$. 

%% file: SG.tex
\section{Verifying assumptions for the sub-Gaussian family}~\label{app:SG}

First, we show a general result for the sub-Gaussian family of distribution that implies Assumption~\ref{ass:Link}.

\begin{lemma}~\label{lem:SG}
Fix $\alpha\in (0,1)$. For any two sub-Gaussian measures $\mu$ and $\nu$ with sub-Gaussian parameter $\sigma_{\mu}$ and $\sigma_{\nu}$ respectively, such that $\nu$ is absolutely continuous wrt $\mu$, the $\alpha-$R\'enyi divergence can be bounded below by the absolute difference between the respective means. In particular, for any random variable $X$ having measure $\mu$ and $\nu$, we have  
   \(     |\E_{\nu}[X]-\E_{\mu}[X]  | \leq \sqrt{(\sigma_{\mu}^2\alpha+\sigma_{\nu}^2(1-\alpha))} \sqrt{\frac{2}{\alpha} D_{\alpha}(\nu\|\mu) }.\)
\end{lemma}

\proof{Proof of Lemma~\ref{lem:SG}:}
    For any $\alpha\in (0,1)$, recall the definition of $\alpha$-R\'enyi divergence  $D_{\alpha}(\nu\|\mu)=   \log \int g^{\alpha}d\mu =\frac{1}{\alpha-1} {\log \int \left(\frac{d\nu}{d\mu}\right)^{\alpha}d\mu}   $, where $g\equiv \frac{d\nu}{d\mu}$. Now observe that
    \begin{align}
    \nonumber
       (\alpha-1) D_{\alpha}(\nu\| \mu) &= \log \E_{\nu}[ g^{\alpha-1} e^{-(\alpha-1)X} e^{(\alpha-1)X}  ] = \log \int (ge^{-X})^{\alpha-1}  e^{(\alpha-1)X} d\nu 
       \\
       \nonumber
       &\leq \log \left( \int   e^{(\alpha-1)X} d\nu \right)^{\alpha} \left(\int (ge^{-X})^{-1}  e^{(\alpha-1)X} d\nu \right)^{1-\alpha}
       \\
       \nonumber
       &= \log \left( \E_{\nu}[  e^{(\alpha-1)X}]^{\alpha} \E_{\nu}[g^{-1}  e^{\alpha X} ]^{1-\alpha} \right)
       \\
       \nonumber
       &= \log \left( \E_{\nu}[  e^{(\alpha-1)X}]^{\alpha} \E_{\mu}[ e^{\alpha X} ]^{1-\alpha} \right)
       \\
       &= \alpha \log \E_{\nu}[  e^{(\alpha-1)X}] + (1-\alpha) \log  \E_{\mu}[ e^{\alpha X} ],
    \end{align}
    where the first inequality follows from the H\"older's inequality wrt measure $e^{(\alpha-1)X} d\nu$.

    Now fix $X=s X$ for any $s\in \R$ (without loss of generality). Since, $\alpha\in(0,1)$, it follows from the inequality above that
    \begin{align}  
        \nonumber
        D_{\alpha}(\nu\|\mu)  &\geq \left[ \frac{\alpha}{\alpha-1}\log \E_{\nu} [e^{(\alpha-1)sX} ] - \log \E_{\mu}[ e^{\alpha sX} ] \right] 
        \\
        \nonumber
        &\geq \left[ \frac{\alpha}{\alpha-1}[s(\alpha-1)\E_{\nu}[X]+\sigma_{\nu}^2s^2(\alpha-1)^2/2] - [s\alpha\E_{\mu}[X]+\sigma_{\mu}^2s^2\alpha^2/2]\right] 
        \\
        &= \left[ s\alpha[\E_{\nu}[X]-\E_{\mu}[X]] - \frac{s^2\alpha}{2}[\sigma_{\nu}^2(1-\alpha)+\sigma_{\mu}^2\alpha]\right] ,
    \end{align}
    where the second inequality uses the fact that $\mu$ and $\nu$ are sub-Gaussian measures.
    Recall if $X\sim\mu$, is a sub-Gaussian random variable, then $\E_{\mu}[e^{sX}] \leq e^{s\E_{\mu}[X]+\sigma_{\mu}^2s^2/2}$ for all $s\in \R$.
    Now it follows from above that  $|\E_{\nu}[X]-\E_{\mu}[X]  | \leq \inf_{|s|} \frac{D_{\alpha}(\nu\|\mu)}{|s|\alpha} + \frac{|s|}{2}(\sigma_{\mu}^2\alpha+\sigma_{\nu}^2(1-\alpha))$.
\halmos{} \endproof

Next, we provide a technical result that bounds 2-R\'enyi divergence between two random variables that are modeled with additive noise.

\begin{lemma}~\label{lem:QuadApp}
    For a given $a\in \sA$ and any $\theta \in \Theta$, let $X = g(a^\top\theta) + \eta$, where $g(\cdot)$ is a known twice differentiable mapping and $\eta$ is a random variable with density $p(\cdot)$. Then for any $\e>0$, there exist a $\theta^*(\epsilon)$ in $B(\theta_0,\e)$, a convex ball centred at $\theta_0$ with radius $\e$, such that for any $\theta\in B(\theta_0,\e)$,
       \( D_2(\theta,\theta_0) = (\theta-\theta_0)^\top a^\top \gI(\theta^*(\e))a (\theta-\theta_0),\)
    where 
   $\gI(\theta ) = \E_{\theta} \left[ ( \nabla_{u} \log p(x-g(u)) ) ( \nabla_{u}\log p(x-g(u)) )^\top \Big |_{u = a^\top \theta}   \right] $ and $\gI(\theta_0)a a^\top$ is the Fisher information about $\theta_0$ that $X$ contains. 
 \end{lemma}
 \proof{Proof:}
     Recall from the definition of the $2-$R\'enyi divergence that, for a given $a \in \sA$,
     \begin{align}
     \nonumber
         { D_2(\theta,\theta_0)} &= \log   \int p_{\theta}(x|a)^2p_{\theta_0}(x|a)^{-1}dx   
         \\ 
         &= \log \E_{\theta_0} \left[ \frac{p_{\theta}(x|a)^2}{p_{\theta_0}(x|a)^2} \right] = \log \E_{\theta_0} \left[ \frac{p(x-g(a^\top \theta))^2}{p(x-g(a^\top \theta_0))^2} \right].
     \end{align}
     For brevity, we denote $\dot{f}$ and $\ddot{f}$ as the first and second derivative of $f$ for $f=\{g,p\}$. It is straightforward to compute the gradient of $D_2$ with respect to $\theta$ as
     \begin{align}
         \nabla_{\theta} D_2 = - \left( \E_{\theta_0} \left[ \frac{p(x-g(a^\top \theta))^2}{p(x-g(a^\top \theta_0))^2} \right]\right)^{-1}   \E_{\theta_0} \left[  \frac{2 p(x-g(a^\top \theta))\dot{p}(x-g(a^\top \theta))\dot{g}(a^\top \theta) a^\top }{p(x-g(a^\top \theta_0))^2}  \right]. 
     \end{align}
     Similarly, the Hessian of $D_2$ is computed as, 
     \begin{align}
         \nonumber
         \nabla_{\theta}^2 D_2 &=  -\left( \E_{\theta_0} \left[ \frac{p(x-g(a^\top \theta))^2}{p(x-g(a^\top \theta_0))^2} \right]\right)^{-2}   \left(\E_{\theta_0} \left[  \frac{2 p(x-g(a^\top \theta))\dot{p}(x-g(a^\top \theta))\dot{g}(a^\top \theta)  }{p(x-g(a^\top \theta_0))^2}  \right]\right)^2 a a^\top 
         \\
         \nonumber
         &+ \left( \E_{\theta_0} \left[ \frac{p(x-g(a^\top \theta))^2}{p(x-g(a^\top \theta_0))^2} \right] \right)^{-1} 
         \\
         &  \E_{\theta_0} \Bigg[ \frac{
         \begin{matrix}
         2 p(x-g(a^\top \theta))\ddot{p}(x-g(a^\top \theta)) \dot{g}(a^\top \theta)^2 + 2 \left( \dot{p}(x-g(a^\top \theta)) \dot{g}(a^\top \theta) \right)^2 \\ 
         - 2 p(x-g(a^\top \theta))\dot{p}(x-g(a^\top \theta))\ddot{g}(a^\top \theta) 
         \end{matrix}}{p(x-g(a^\top \theta_0))^2}  \Bigg] a a^\top.
     \end{align}
      When we evaluate the above Hessian at $\theta=\theta_0$ using the fact that $\nabla_{\theta}\E_{\theta_0}[\log p(x-g(a^\top\theta))] \Big|_{\theta=\theta_0} = 0$ and $\E_{\theta_0} \left[ \frac{2  \nabla_{\theta}^2 {p}(x-g(a^\top \theta))  }{p(x-g(a^\top \theta_0))} \right] = \nabla_{\theta}^2 \int p_{\theta}(x|a)dx = 0$, then 
     \begin{align}
         \nabla_{\theta}^2 D_2 \Bigg| _{\theta= \theta_0} &=   \E_{\theta_0} \left[  2  ( \nabla_{u} \log p(x-g(u)) ) ( \nabla_{u}\log p(x-g(u)) )^\top \Big |_{u = a^\top \theta_0}   \right] a a^\top 
          =     2 \gI(\theta_0)a a^\top.
     \end{align}
     Now using second-order Taylor's theorem, it follows that for any $\e>0$, there exists a $\theta^*(\epsilon)$ in $B(\theta_0,\e)$ such that for any $\theta\in B(\theta_0,\e)$,
       \(  D_2(\theta,\theta_0) =  2 (\theta-\theta_0)^\top a^\top \gI(\theta^*(\e))a(\theta-\theta_0).\)
 \halmos{} \endproof
 
Our next result shows that the sub-Gaussian family satisfies Assumption~\ref{ass:prior}.
\begin{lemma}~\label{lem:SG_prior}
    Under Assumption~\ref{ass:action}, the sub-Gaussian family of distributions as defined in~Definition~\ref{ass:reward}, satisfies Assumption~\ref{ass:prior} for $\e_t^2= \frac{4d \tilde{C}\log(t)}{D\alpha  C_{\phi}t} $ and sufficiently large $t$ ($t\geq \log(t)/C_{\phi}$), where $\tilde{C} \geq \gI(\theta)$ for any $\theta\in B(\theta_0,1)$ and $C_{\phi} = \min_{i\in[d], x\in B(\theta_{0}^i,1)}\phi_i (x)$,  $B(\theta_{0}^i,u)\subset \R$ is a ball of radius $u$ centered at $\theta_{0}^i\in \Theta$ and $\phi_i(\cdot)$ is the density of $\Pi^i_t$.
\end{lemma}
\proof{Proof of Lemma~\ref{lem:SG_prior}:}
        Using Lemma~\ref{lem:QuadApp}, there exist a $\theta^*(1) \in B(\theta_0,1)$ (defined later), a convex ball centred at $\theta_0$ with radius $1$, such that for any $\theta\in B( \theta_0,1)$, 
    \begin{align}
        \nonumber
        D^{(t)}_{2}(\theta, \theta_0) &=  \log \int p_{\theta}^{(t)}(r^{(t)}|a^{(t)})^{2} p_{0}^{(t)}(r^{(t)}|a^{(t)})^{-1} d\mu^{(t)}
        \\
        \nonumber
        &= \sum_{s=1}^{t}  \log \int p_{\theta}(r_s| \gF_{t-1},a_s)^{2} p_{0}(r_s| \gF_{s-1},a_s)^{-1} d\mu
        \\
        &\leq  \tilde C(\theta-\theta_0)^\top \left(  \sum_{s=1}^{t}  a_s a_s^\top \right) (\theta-\theta_0)
        \leq t  \tilde C  \|\theta-\theta_0\|^2.
        \label{eq:SG1}
    \end{align}
    where in the first inequality $\tilde C$ is the bound on $\gI(\theta)$ for any $\theta\in B(\theta_0,1)$ and the last inequality follows due to CS and~Assumption~\ref{ass:action}. Now observe for $B(\theta_0,1)= \{\theta \in \Theta : \|\theta-\theta_0)\| \leq 1 \}$ that
\begin{align}
    \nonumber
    \Pi \left(D^{(t)}_{2}(\pmb \theta, \pmb \theta_0) \leq \frac{D\alpha}{4}t\e_t^2  \right) &\geq \Pi \left( B(\theta_0,1) ,D^{(t)}_{2}(\theta, \theta_0) \leq \frac{D\alpha}{4}t\e_t^2  \right)
    \\
    \nonumber
    &\geq \Pi \left(  B(\theta_0,1), \|\theta-\theta_0\|^2  \leq \frac{D\alpha}{4\tilde {C}}\e_t^2  \right)
    \\
    \nonumber
    &=\Pi \left( (\theta-\theta_0)^\top (\theta-\theta_0) \leq \frac{D\alpha }{4 \tilde {C}}\e_t^2  \right)
    \\
    &\geq \prod_{i=1}^{d} \Pi^i\left( (\theta^i-\theta_0^i)^2 \leq \frac{D \alpha}{4d \tilde {C} }\e_T^2  \right),
    \label{eq:LB0}
\end{align}
where the second inequality follows from~\eqref{eq:SG1} and the first equality is due to the assumption that $\frac{D\alpha } {4 \tilde {C}d } \e_t^2<1$ (for sufficiently large $t$). Now the result follows for $\e_t^2 = \frac{4d\tilde{C} \log(t)}{D\alpha t C_{\phi}}$.

\halmos{} \endproof

To show that the sub-Gaussian family satisfies Assumption~\ref{ass:Spokoiny}, we first write down various expressions used in their definition for this family.
The conditional log-likelihood of generating $X_t|a^{(t)}$ as $\sL(\theta) = \log \prod_{s=1}^{t} \left[p_{\eta}(r_s-a_s^\top\theta)\right] $, where $p_{\eta}$ is the density of the sub-Gaussian error with sub-Gaussian parameter $1$. Also, denote $\log p_{\eta} := h_{\eta}$.
The stochastic part of the conditional log-likelihood is denoted as $\zeta (\theta):= \sL(\theta) - \E[\sL(\theta)|a^{(t)}]$ and $\nabla \zeta (\theta) = \nabla \sL(\theta) - \E[\nabla \sL(\theta)|a^{(t)}]= \sum_{s=1}^{t} -h_{\eta}'(r_s-a_s^\top\theta)a_s^\top +\E[ h_{\eta}'(r_s-a_s^\top\theta)a_s^\top|a^{(t)}] $ and therefore $\E[\nabla \zeta(\theta_0)|a^{(t)}]=- \sum_{s=1}^{t} h_{\eta}'(r_s-a_s^\top\theta)a_s^\top$ (since $\E[\nabla \sL(\theta_0)|a^{(t)}]=0$). 
We assume that $h_\eta$ is twice continuously differentiable and let
${\BFh}^2:= -\int h_\eta''(z)p_{\eta}(z)dz <\infty$.
Also, 
\begin{align}
\sD_0^2 = \sD_0^{2}(\theta_0)
     = 
    - \nabla^{2} \E[ \sL(\theta_0)|a^{(t)}] = -\sum_{s=1}^{t} \E[h_{\eta}''(r_s-a_s^\top\theta_0)|a^{(t)}]a_s a_s^\top = {\BFh}^2 \sum_{s=1}^{t} a_s a_s^\top.
    \end{align}
Similarly, $\sD_0^{2}(\theta) = -\sum_{s=1}^{t} \E[h_{\eta}''(r_s-a_s^\top\theta)|a^{(t)}]a_s a_s^\top = -\sum_{s=1}^{t} \E[h_{\eta}''(\eta + a_s^\top(\theta_0-\theta))|a^{(t)}]a_s a_s^\top $.
 Henceforth, use $\E_a$ in place of the conditional expectation $\E[\cdot|a^{(t)}]$.   

Next, we assume two other conditions on the error density $p_{\eta}$.
\begin{assumption}~\label{ass:AddSG}
\begin{enumerate}
    \item There exists some constant $v_0$ and ${\BFg}_1>0$, such that a random variable $\eta \sim p_\eta$, it holds that
    \(    \log \E \exp(\mu h'_{\eta}(\eta)/ {\BFh}) \leq v_0^2\mu^2/2, \quad  |\mu|<{\BFg}_1.\)
    \item There exists some constant $v_0$ and for every ${\BFr}>0$, there exists ${\BFg}_1({\BFr})>0$, such that for all $\delta$ with $|\delta|<\sT_2^{-1/2}{\BFr}/\vartheta_s$ it holds that
    \(   \log \E \left[\exp
    \left(\frac{\mu}{\vartheta_s^2} \left\{h''_{\eta}(\eta_s+\delta) -\E[h''_{\eta}(\eta_s+\delta)] \right\}\right) \right] \leq v_0^2\mu^2/2, \quad  |\mu|<{\BFg}_1({\BFr}),
    \)
    where $\vartheta_s$ is known value (given $a^{(t)}$) and $\sT_2^{-1/2}:= \max_s \sup_{\gamma \in \R^d} \frac{\vartheta_s |a_s^\top \gamma|}{\|\sD_0\gamma\|}$.
\end{enumerate}
    
\end{assumption}

The above assumption essentially requires that the error distribution has an exponentially decaying tail. Since we assume that the error distribution is sub-Gaussian due to Assumption~\ref{ass:reward}, the above condition is automatically satisfied.

\begin{lemma}~\label{lem:SpokSG}
    The sub-Gaussian family of distributions (Definition~\ref{ass:reward}) satisfies Assumption~\ref{ass:Spokoiny}. The condition $(ED_0)$ is satisfied with ${\BFg}:= {\BFg}_1 \sT^{1/2}$, where $\sT^{-1/2}:= \max_{s\in[t]} \sup_{\gamma \in \R^d} \frac{A''(a_s^\top \theta_0)^{1/2} |a_s^\top \gamma|}{\|\sD_0 \gamma\|} $ and condition $(ED_1)$ can be satisfied for any $\omega>0$ and ${\BFg}>0$. Condition $\gL_0$ follows for $\delta({\BFr})= L \sT_2^{-1/2} {\BFr}$, where $\sT^{-1/2}:= \max_{s\in[t]} \sup_{\gamma \in \R^d} \frac{A''(a_s^\top \theta_0)^{1/2} |a_s^\top \gamma|}{\|\sD_0 \gamma\|} $. 
\end{lemma}
\proof{Proof:}
    We derive all the conditions in seriatim. The proofs are adapted from~\cite{Panov2015}.
  \begin{enumerate}
    \item Proof for \({(E\!D_{0})} \):
        Using the definition of $\nabla \zeta(\theta_0)$ and $\sD_0$, observe that
        \begin{align}
            \E_a \exp\left\{
              {\BFm} \frac{\langle \nabla \zeta(\theta_0),\gamma \rangle}
              {\| \sD_0 \gamma \|}
              \right\} &= \E_a \exp\left\{-
              {\BFm} \frac{ \sum_{s=1}^{t} h_{\eta}'(r_s-a_s^\top\theta) a_s^\top\gamma}
              {\| \sD_0 \gamma \|}
              \right\}
              \\
              &= \E_a \exp\left\{ \frac{ -
              {\BFm} {\BFh}\sum_{s=1}^{t}  a_s^\top\gamma}
              {\| \sD_0 \gamma \|} \frac{h_{\eta}'(\eta)}{{\BFh}}
              \right\}.
        \end{align}
        
        Define $\sT^{-1/2}:= \max_s \sup_{\gamma\in \R^d} \frac{{\BFh}|a_s^\top \gamma|}{\| \sD_0 \gamma \|} \geq \frac{{\BFh}|a_s^\top \gamma|}{\| \sD_0 \gamma \|}  $ and ${\BFg}= {\BFg}_1\sT^{1/2}$. Consequently, for $\mu=\frac{-
              {\BFm} {\BFh}  a_s^\top\gamma}
              {\| \sD_0 \gamma \|}$, $|\mu|\leq |{\BFm}|\frac{
               {\BFh}  |a_s^\top\gamma|}
              {\| \sD_0 \gamma \|}<{\BFg} \sT^{1/2} ={\BFg_1}$. Therefore, using Assumption~\ref{ass:AddSG}, the result follows as
              \begin{align}
                  \E_a \exp\left\{ \frac{ -
              {\BFm} {\BFh}\sum_{s=1}^{t}  a_s^\top\gamma}
              {\| \sD_0 \gamma \|} \frac{h_{\eta}'(\eta)}{{\BFh}}.
              \right\} \leq \exp\left(\frac{v_0^2 {\BFm}^2}{2}\frac{
               {\BFh}^2  \sum_{s=1}^t |a_s^\top\gamma|^2}
              {\| \sD_0 \gamma \|^2}\right) = \exp\left(\frac{v_0^2 {\BFm}^2}{2}\right).
              \end{align}
    \item Proof for \( {(E\!D_{1})} \):
        Using the definition of $\zeta(\theta)$, observe that
        \begin{align}
        \nonumber
            \E_a &\exp\left\{
        \frac{{\BFm}}{\omega} \frac{\gamma_1^{\top}\nabla^2 \zeta(\theta)\gamma_2 }{ \|\sD_0 \gamma_1\|\|\sD_0 \gamma_2\|}
        \right\} 
        \\
        \nonumber
        &=\E_a \exp\left\{
        \frac{{\BFm}}{\omega} \frac{ \sum_{s=1}^{t} \left( h_{\eta}''(r_s-a_s^\top\theta) -\E_a[ h_{\eta}''(r_s-a_s^\top\theta)] \right)\gamma_1^\top a_s a_s^\top \gamma_2 }{ \|\sD_0 \gamma_1\|\|\sD_0 \gamma_2\|}
        \right\} 
        \\
        &=\E_a \exp\left\{
        \frac{{\BFm}}{\omega} \frac{ \sum_{s=1}^{t} \left( h_{\eta}''(\eta_s-a_s^\top(\theta-\theta_0)) -\E_a[ h_{\eta}''(\eta_s-a_s^\top(\theta-\theta_0))] \right)\gamma_1^\top a_s a_s^\top \gamma_2 }{ \|\sD_0 \gamma_1\|\|\sD_0 \gamma_2\|}
        \right\}
        .
        \end{align}
        Using the definition of $\sT_2$ for $\theta\in \Theta_0({\BFr})$, observe that $\sT_2^{-1/2} \geq \frac{\vartheta_s|a_s^\top(\theta-\theta_0)|}{\|\sD_0(\theta-\theta_0)\|}$. Therefore, $|a_s^\top(\theta-\theta_0)|\leq \sT_2^{-1/2}\frac{\|\sD_0(\theta-\theta_0)\|}{\vartheta_s} \leq \sT_2^{-1/2}\frac{{\BFr}}{\vartheta_s}$. For $\mu = \frac{{\BFm}\vartheta_s^2 \gamma_1^\top a_s a_s^\top \gamma_2 }{\omega \|\sD_0 \gamma_1\|\|\sD_0 \gamma_2\| }$, observe that $|\mu| \leq \frac{|m|}{\omega}\frac{\vartheta_s|a_s^\top\gamma_1|}{\|\sD_0 \gamma_1\|}\frac{\vartheta_s|a_s^\top\gamma_2|}{\|\sD_0 \gamma_2\|} \leq \frac{{\BFg}(r)}{\omega}\sT^{-1}:= {\BFg}_1({\BFr})$. It follows from Assumption~\ref{ass:AddSG}(2), that
        \begin{align}
            \E_a \exp\left\{
        \frac{{\BFm}}{\omega} \frac{\gamma_1^{\top}\nabla^2 \zeta(\theta)\gamma_2 }{ \|\sD_0 \gamma_1\|\|\sD_0 \gamma_2\|}
        \right\}  \leq \exp\left(\sum_{s=1}^t v_0^2 \frac{{\BFm}^2\vartheta_s^4 |\gamma_1^\top a_s|^2 |a_s^\top \gamma_2|^2 }{2\omega^2 \|\sD_0 \gamma_1\|^2\|\sD_0 \gamma_2\|^2 } \right) \leq \exp\left( \frac{v_0^2{\BFm}^2}{2\omega^2} \frac{t}{\sT_2^2} \right).
        \end{align}
        By fixing $\omega=\frac{\sqrt{t}}{\sT_2}$, the result follows.
    \item Proof for \({(\gL_0)} \):
      For $I_d= \sD_0^{-1} \sD_0^2\sD_0^{-1}$, we have by the definition of the operator norm
      \begin{align}
      \nonumber
          &\|\sD_0^{-1} (\sD_0^2(\theta)-\sD_0^2)\sD_0^{-1}\|
          \\
          \nonumber
          &=\sup_{\gamma\in \R^d:\|\gamma\|=1} |\gamma^\top  \sD_0^{-1} (\sD_0^2(\theta)-\sD_0^2)\sD_0^{-1} \gamma|
          \\
          \nonumber
          &= \sup_{\gamma\in \R^d:\|\gamma\|=1} \left|\sum_{s=1}^{t} ( -\E[h_{\eta}''(\eta + a_s^\top(\theta_0-\theta))|a^{(t)}] + \E[h_{\eta}''(\eta)|a^{(t)}]) \gamma^\top  \sD_0^{-1}a_s a_s^\top\sD_0^{-1} \gamma\right|
          \end{align}
          \begin{align}
          \nonumber
          &\leq \sup_{\gamma\in \R^d:\|\gamma\|=1} \sum_{s=1}^{t} \E\left[\left|h_{\eta}''(\eta) -h_{\eta}''(\eta + a_s^\top(\theta_0-\theta))\right| |a^{(t)}\right]  \gamma^\top  \sD_0^{-1}a_s a_s^\top\sD_0^{-1} \gamma
          \\
          &\leq \sup_{\gamma\in \R^d:\|\gamma\|=1} \sum_{s=1}^{t} L \left|a_s^\top (\theta -  \theta_0)\right| \gamma^\top  \sD_0^{-1}a_s a_s^\top\sD_0^{-1} \gamma.
      \end{align}
      Define $\sT_1^{-1/2} : = \max_s \sup_{\gamma \in \R^d} \frac{{\BFh}|a_s^\top \gamma|}{\|\sD_0\gamma\|} \geq \frac{{\BFh}|a_s^\top (\theta-\theta_0)|}{\|\sD_0(\theta-\theta_0)\|}  $ and observe that
      \begin{align}
      \nonumber
          \|\sD_0^{-1} (\sD_0^2(\theta)-\sD_0^2)\sD_0^{-1}\|
          &\leq \sup_{\gamma\in \R^d:\|\gamma\|=1} \sum_{s=1}^{t} L \left|a_s^\top (\theta -  \theta_0)\right| \gamma^\top  \sD_0^{-1}a_s a_s^\top\sD_0^{-1} \gamma
          \\
          \nonumber
          &\leq 
          \frac{L \sT_1^{-1/2}}{{\BFh}} \|\sD_0(\theta-\theta_0)\| \sup_{\gamma\in \R^d:\|\gamma\|=1}  \gamma^\top  \sD_0^{-1} \sum_{s=1}^{t} {\BFh}^2 a_s a_s^\top\sD_0^{-1} \gamma 
          \\
          &= \frac{L}{\sT_1^{1/2}{\BFh}} {\BFr}.
      \end{align}
      Consequently, $\delta({\BFr})= L {\BFh}^{-1} \sT_2^{-1/2} {\BFr}. $

    \item Proof for \( {(\gL{{\BFr}})} \):
          Note that 
          \begin{align}
          \nonumber
              \E_a[\sL(\theta,\theta_0)]&= \E_a[\sL(\theta)-\sL(\theta_0)]  
              \\
              \nonumber
              =\sum_{s=1}^{t}\E_a&\left[\log\left(\frac{p_{\theta}(r_s|a_s)}{p_{\theta_0}(r_s|\gF_{s-1},a_s)}\right)\right]
              \\
              \nonumber
              = \sum_{s=1}^{t} \E_a&\left[ h_{\eta}(r_s-a_s^\top \theta)- h_{\eta}(r_s-a_s^\top \theta_0)\right] 
              \\
              \nonumber
              =  \sum_{s=1}^{t} \E_a&\left[ -h_{\eta}'(r_s-a_s^\top \theta_0)a_s^\top (\theta-\theta_0) + h_{\eta}''(r_s-a_s^\top \theta^*)(\theta-\theta_0)^\top a_s a_s^\top(\theta-\theta_0) \right]
              \\
              = \sum_{s=1}^{t} \E_a&[h_{\eta}''(r_s-a_s^\top \theta^*)](\theta-\theta_0)a_s a_s^\top(\theta-\theta_0) = -\|\sD_0(\theta^*)(\theta-\theta_0)\|^2,
          \end{align}
          where the third equality uses second-order Taylor's theorem for a $\theta^*$ that lies between $\theta$ and $\theta_0$ and the penultimate inequality uses the fact that $\sum_{s=1}^{t} \E_a\left[ -h_{\eta}'(r_s-a_s^\top \theta_0)a_s^\top\right]= \nabla\E_a[\sL(\theta_0)]=0$.

          Now observe that $\frac{- \E_a[\sL(\theta,\theta_0)] }{{\BFr}^2}= \frac{\|\sD_0(\theta^*)(\theta-\theta_0)\|^2}{\|\sD_0(\theta-\theta_0)\|^2}  $. Note that, by definition $\theta^*$ is a mapping from ${\BFr}$ and so is $\frac{\|\sD_0(\theta^*)(\theta-\theta_0)\|^2}{\|\sD_0(\theta-\theta_0)\|^2}$. Therefore, we can always find a mapping ${\BFb}({\BFr})$ such that  ${\BFr}{\BFb}({\BFr}) \to \infty$  as ${\BFr} \to \infty$ and $\frac{- \E_a[\sL(\theta,\theta_0)] }{{\BFr}^2}> {\BFr}^2 {\BFb}({\BFr})$.  \halmos{} 
  \end{enumerate} 
\endproof
  $\Delta_t(d,\eta)<1$ can be established using similar steps as used for exponential family models for ${\BFr}_0^2=C(d+\eta) $~\cite[Theorem 6]{Panov2015} for sufficiently large $t_1\geq1$ and $\eta=O(d)$. 

%% file: refs.bib
@inproceedings{kim2023double,
  title={Double doubly robust thompson sampling for generalized linear contextual bandits},
  author={Kim, Wonyoung and Lee, Kyungbok and Paik, Myunghee Cho},
  booktitle={Proceedings of the AAAI Conference on Artificial Intelligence},
  volume={37-7},
  pages={8300--8307},
  year={2023}
}

@misc{agrawal2014thompsonsamplingcontextualbandits,
      title={Thompson Sampling for Contextual Bandits with Linear Payoffs}, 
      author={Shipra Agrawal and Navin Goyal},
      year={2014},
      eprint={1209.3352},
      archivePrefix={arXiv},
      primaryClass={cs.LG},
      url={https://arxiv.org/abs/1209.3352}, 
}

@inproceedings{chakraborty2023thompson,
  title={Thompson sampling for high-dimensional sparse linear contextual bandits},
  author={Chakraborty, Sunrit and Roy, Saptarshi and Tewari, Ambuj},
  booktitle={International Conference on Machine Learning},
  pages={3979--4008},
  year={2023},
  organization={PMLR}
}

@article{rigollet2010nonparametric,
  title={Nonparametric Bandits with Covariates},
  author={Rigollet, Philippe and Zeevi, Assaf},
  journal={COLT 2010},
  pages={54},
  year={2010},
  publisher={Citeseer}
}

@article{greenewald2017action,
  title={Action centered contextual bandits},
  author={Greenewald, Kristjan and Tewari, Ambuj and Murphy, Susan and Klasnja, Predag},
  journal={Advances in neural information processing systems},
  volume={30},
  year={2017}
}

@article{kim2021multi,
  title={MULTI-ARMED BANDITS WITH COVARIATES},
  author={Kim, Dong Woo and Lai, Tze Leung and Xu, Huanzhong},
  journal={Statistica Sinica},
  volume={31},
  pages={2275--2287},
  year={2021},
  publisher={JSTOR}
}

@article{bastani2020online,
  title={Online decision making with high-dimensional covariates},
  author={Bastani, Hamsa and Bayati, Mohsen},
  journal={Operations Research},
  volume={68},
  number={1},
  pages={276--294},
  year={2020},
  publisher={INFORMS}
}

@article{Nelder1972,
  title = {Generalized Linear Models},
  volume = {135},
  ISSN = {0035-9238},
  url = {http://dx.doi.org/10.2307/2344614},
  DOI = {10.2307/2344614},
  number = {3},
  journal = {Journal of the Royal Statistical Society. Series A (General)},
  publisher = {JSTOR},
  author = {Nelder,  J. A. and Wedderburn,  R. W. M.},
  year = {1972},
  pages = {370}
}

@book{van2000asymptotic,
	author = {Van der Vaart, Aad W},
	publisher = {Cambridge university press},
	title = {Asymptotic statistics},
	volume = {3},
	year = {2000}}

@article{urteaga2018nonparametric,
  title={Nonparametric gaussian mixture models for the multi-armed contextual bandit},
  author={Urteaga, I{\~n}igo and Wiggins, Chris H},
  journal={stat},
  volume={1050},
  pages={8},
  year={2018}
}

@InProceedings{hong22b,
  title = 	 { Thompson Sampling with a Mixture Prior },
  author =       {Hong, Joey and Kveton, Branislav and Zaheer, Manzil and Ghavamzadeh, Mohammad and Boutilier, Craig},
  booktitle = 	 {Proceedings of The 25th International Conference on Artificial Intelligence and Statistics},
  pages = 	 {7565--7586},
  year = 	 {2022},
  editor = 	 {Camps-Valls, Gustau and Ruiz, Francisco J. R. and Valera, Isabel},
  volume = 	 {151},
  series = 	 {Proceedings of Machine Learning Research},
  month = 	 {28--30 Mar},
  publisher =    {PMLR},
  url = 	 {https://proceedings.mlr.press/v151/hong22b.html}
}

@article{abbasi2011improved,
  title={Improved algorithms for linear stochastic bandits},
  author={Abbasi-Yadkori, Yasin and P{\'a}l, D{\'a}vid and Szepesv{\'a}ri, Csaba},
  journal={Advances in neural information processing systems},
  volume={24},
  year={2011}
}

@misc{hamidi_frequentist_2023,
	title = {On {Frequentist} {Regret} of {Linear} {Thompson} {Sampling}},
	url = {http://arxiv.org/abs/2006.06790},
	language = {en},
	urldate = {2024-01-10},
	publisher = {arXiv},
	author = {Hamidi, Nima and Bayati, Mohsen},
	month = apr,
	year = {2023},
	note = {arXiv:2006.06790 [cs, stat]},
}

@misc{luo_geometry-aware_2023,
	title = {Geometry-{Aware} {Approaches} for {Balancing} {Performance} and {Theoretical} {Guarantees} in {Linear} {Bandits}},
	url = {http://arxiv.org/abs/2306.14872},
	language = {en},
	urldate = {2024-01-10},
	publisher = {arXiv},
	author = {Luo, Yuwei and Bayati, Mohsen},
	month = dec,
	year = {2023},
	note = {arXiv:2306.14872 [cs, stat]},
}

@inproceedings{abeille2017linear,
  title={Linear thompson sampling revisited},
  author={Abeille, Marc and Lazaric, Alessandro},
  booktitle={Artificial Intelligence and Statistics},
  pages={176--184},
  year={2017},
  organization={PMLR}
}

@article{Panov2015,
  doi = {10.1214/14-ba926},
  url = {https://doi.org/10.1214/14-ba926},
  year = {2015},
  month = sep,
  publisher = {Institute of Mathematical Statistics},
  volume = {10},
  number = {3},
  author = {Maxim Panov and Vladimir Spokoiny},
  title = {Finite Sample Bernstein {\textendash} von Mises Theorem for Semiparametric Problems},
  journal = {Bayesian Analysis}
}

@article{Spokoiny2012,
  doi = {10.1214/12-aos1054},
  url = {https://doi.org/10.1214/12-aos1054},
  year = {2012},
  month = dec,
  publisher = {Institute of Mathematical Statistics},
  volume = {40},
  number = {6},
  author = {Vladimir Spokoiny},
  title = {Parametric estimation. Finite sample theory},
  journal = {The Annals of Statistics}
}

@misc{carpentier_elliptical_2020,
	title = {The {Elliptical} {Potential} {Lemma} {Revisited}},
	url = {http://arxiv.org/abs/2010.10182},
	urldate = {2024-01-10},
	publisher = {arXiv},
	author = {Carpentier, Alexandra and Vernade, Claire and Abbasi-Yadkori, Yasin},
	month = oct,
	year = {2020},
	note = {arXiv:2010.10182 [cs, stat]},
}

@misc{jaiswal2022,
	doi = {10.XXXXX/ARXIV.1705.XXXXX},
	
	url = {https://arxiv.org/abs/1705.XXXXX},
	
	author = {Jaiswal, Prateek and Pati, Debdeep and  Bhattacharya, Anirban and Mallick, Bani },
	
	title = {Generalized Regret Analysis of Thomson Sampling with Fractional Posteriors},
	
	publisher = {arXiv},
	
	year = {2022},
	
	copyright = {arXiv.org perpetual, non-exclusive license}
}

@article{Zhang2021feel,
	title={Feel-good thompson sampling for contextual bandits and reinforcement learning},
	author={Zhang, Tong},
	journal={arXiv preprint arXiv:2110.00871},
	year={2021}
}

@article{Robins1952,
author = {Herbert Robbins},
title = {{Some aspects of the sequential design of experiments}},
volume = {58},
journal = {Bulletin of the American Mathematical Society},
number = {5},
publisher = {American Mathematical Society},
pages = {527 -- 535},
year = {1952},
doi = {bams/1183517370},
URL = {https://doi.org/}
}

@inproceedings{Dani2008,
  title={Stochastic Linear Optimization under Bandit Feedback},
  author={Varsha Dani and Thomas P. Hayes and Sham M. Kakade},
  booktitle={COLT},
  year={2008}
}

@article{ZG,
	doi = {10.1214/19-aos1883},
	url = {https://doi.org/10.1214/19-aos1883},
	year = {2020},
	month = aug,
	publisher = {Institute of Mathematical Statistics},
	volume = {48},
	number = {4},
	author = {Fengshuo Zhang and Chao Gao},
	title = {Convergence rates of variational posterior distributions},
	journal = {The Annals of Statistics}
}

@article{GGV,
	Author = {Subhashis Ghosal and Jayanta K. Ghosh and Aad W. van der Vaart},
	Issn = {00905364},
	Journal = {Ann. Statist.},
	Number = {2},
	Pages = {500--531},
	Publisher = {Institute of Mathematical Statistics},
	Title = {Convergence Rates of Posterior Distributions},
	Url = {http://www.jstor.org/stable/2674039},
	Volume = {28},
	Year = {2000}}

@inproceedings{li2012open,
	title={Open problem: Regret bounds for thompson sampling},
	author={Li, Lihong and Chapelle, Olivier},
	booktitle={Conference on Learning Theory},
	pages={43--1},
	year={2012},
	organization={JMLR Workshop and Conference Proceedings}
}

@article{chapelle2011empirical,
	title={An empirical evaluation of thompson sampling},
	author={Chapelle, Olivier and Li, Lihong},
	journal={Advances in neural information processing systems},
	volume={24},
	year={2011}
}

@article{Auer2002,
  title={Finite-time analysis of the multiarmed bandit problem},
  author={Auer, Peter and Cesa-Bianchi, Nicolo and Fischer, Paul},
  journal={Machine learning},
  volume={47},
  number={2},
  pages={235--256},
  year={2002},
  publisher={Springer}
}

@article{lai1985asymptotically,
	title={Asymptotically efficient adaptive allocation rules},
	author={Lai, Tze Leung and Robbins, Herbert and others},
	journal={Advances in applied mathematics},
	volume={6},
	number={1},
	pages={4--22},
	year={1985}
}

@inproceedings{agrawal2013thompsonLIN,
	title={Thompson sampling for contextual bandits with linear payoffs},
	author={Agrawal, Shipra and Goyal, Navin},
	booktitle={International conference on machine learning},
	pages={127--135},
	year={2013},
	organization={PMLR}
}

@article{THOMPSON1933,
	doi = {10.2307/2332286},
	url = {https://doi.org/10.2307/2332286},
	year = {1933},
	month = dec,
	publisher = {{JSTOR}},
	volume = {25},
	number = {3/4},
	pages = {285},
	author = {William R. Thompson},
	title = {On the Likelihood that One Unknown Probability Exceeds Another in View of the Evidence of Two Samples},
	journal = {Biometrika}
}

@article{bhattacharya2019bayesian,
	title={Bayesian fractional posteriors},
	author={Bhattacharya, Anirban and Pati, Debdeep and Yang, Yun},
	journal={The Annals of Statistics},
	volume={47},
	number={1},
	pages={39--66},
	year={2019},
	publisher={Institute of Mathematical Statistics}
}

@article{van2014renyi,
	title={R{\'e}nyi divergence and Kullback-Leibler divergence},
	author={Van Erven, Tim and Harremos, Peter},
	journal={IEEE Transactions on Information Theory},
	volume={60},
	number={7},
	pages={3797--3820},
	year={2014},
	publisher={IEEE}
}
